\documentclass{article}
\PassOptionsToPackage{numbers,sort&compress}{natbib}

\usepackage[utf8]{inputenc} % allow utf-8 input
\usepackage[T1]{fontenc}    % use 8-bit T1 fonts
\usepackage{hyperref}       % hyperlinks
\usepackage{url}            % simple URL typesetting
\usepackage{booktabs}       % professional-quality tables
\usepackage{amsmath, amssymb, amsthm, amsfonts}       % blackboard math symbols
\usepackage{microtype}      % microtypography
\usepackage{xcolor}         % colors
\usepackage{enumitem}
\usepackage{graphicx}
\usepackage{algorithm}
\usepackage{algorithmic}
\usepackage{natbib}[numbers,sort&compress]
\usepackage{geometry}
\usepackage{authblk,textcomp}
\usepackage{mathtools, bm, xfrac}
\usepackage{wrapfig}
\usepackage{commath}
\usepackage{siunitx}

\usepackage[cal=euler]{mathalfa}
\usepackage{libertine}

\DeclareMathOperator{\Var}{Var}
\DeclareMathOperator{\Cov}{Cov}

\DeclareMathOperator{\Trace}{Tr} 
\DeclareMathOperator{\Diag}{diag} 

\def\dd{\text{d}}
\def\Tr#1{\Trace{\left[#1\right]}}
\def\diag#1{\Diag{\left(#1\right)}}

\def\mat#1{\mathrm{#1}}
\renewcommand{\vec}[1]{\bm{#1}}

\def\w{\vec{w}}
\def\M{\mat{M}}
\def\Q{\mat{Q}}
\def\P{\mat{P}}

\def\inp{d}
\def\hids{p}
\def\hidt{k}
\def\nsamp{n}
\def\i{\nu}
\def\noise{\Delta}
\def\act{\sigma}
\def\lr{\gamma}
\def\datadist{\rho}
\def\dsp{\mathcal{E}}
\def\lf{\lambda}

\def\erf{\text{erf}}

\def\Explf{\mathbb E_{(\vec{\lambda},\vec{\lambda}^{\star})\sim\mathcal{N}(\vec{0}_{\hids+\hidt}, \Omega)}}

\newtheorem{assump}{Assumption}

\hypersetup{pdfauthor={ENS},pdftitle={A unifying approach to SGD},%
            colorlinks, linktocpage=true, pdfstartpage=1, pdfstartview=FitV,%
    breaklinks=true, pdfpagemode=UseNone, pageanchor=true, pdfpagemode=UseOutlines,%
    plainpages=false, bookmarksnumbered, bookmarksopen=true, bookmarksopenlevel=1,%
    hypertexnames=true, pdfhighlight=/O,%
    urlcolor=orange, linkcolor=blue, citecolor=blue
        }

\title{From high-dimensional \& mean-field dynamics to dimensionless ODEs: A unifying approach to SGD in two-layers networks}

\author[1]{Luca Arnaboldi}
\author[1]{Ludovic Stephan}
\author[1]{Florent Krzakala}
\author[2]{Bruno Loureiro}
\affil[1]{\small \'Ecole Polytechnique F\'ed\'erale de Lausanne (EPFL),
  IdePHICS Lab,   CH-1015 Lausanne, Switzerland}
\affil[2]{\small D\'epartement d'Informatique, \'Ecole Normale Sup\'erieure (ENS) - PSL \& CNRS, 
F-75230 Paris cedex 05, France}
\affil[ ]{\textit {\{luca.arnaboldi, ludovic.stephan, florent.krzakala\}@epfl.ch, bruno.loureiro@di.ens.fr}}

\makeatletter
\newtheorem*{rep@theorem}{\rep@title}
\newcommand{\newreptheorem}[2]{%
\newenvironment{rep#1}[1]{%
 \def\rep@title{#2 \ref{##1}}%
 \begin{rep@theorem}}%
 {\end{rep@theorem}}}
\makeatother

\theoremstyle{plain}
\newtheorem{theorem}{Theorem}
\numberwithin{theorem}{section}
\newreptheorem{theorem}{Theorem}

\newtheorem{lemma}[theorem]{Lemma}
\newtheorem{proposition}[theorem]{Proposition}

\theoremstyle{remark}
\newtheorem{remark}[theorem]{Remark}
\date{\today}

\begin{document}
\maketitle

%%%%%%%%%%%%%%%%%%%%%%%%%%%%%%%%%%%%%%%%%%%%%%%%%%%%%%%%%%%%%%%%%%%%%%%%%%%%%%%
\begin{abstract}
This manuscript investigates the one-pass stochastic gradient descent (SGD) dynamics of a two-layer neural network trained on Gaussian data and labels generated by a similar, though not necessarily identical, target function. We rigorously analyse the limiting dynamics via a deterministic and low-dimensional description in terms of the sufficient statistics for the population risk. Our unifying analysis bridges different regimes of interest, such as the classical gradient-flow regime of vanishing learning rate, the high-dimensional regime of large input dimension, and the overparameterised ``mean-field'' regime of large network width, covering as well the intermediate regimes where the limiting dynamics is determined by the interplay between these behaviours. In particular, in the high-dimensional limit, the infinite-width dynamics is found to remain close to a low-dimensional subspace spanned by the target principal directions. Our results therefore provide a unifying picture of the limiting SGD dynamics with synthetic data. 
\end{abstract}

%%%%%%%%%%%%%%%%%%%%%%%%%%%%%%%%%%%%%%%%%%%%%%%%%%%%%%%%%%%%%%%%%%%%%%%%%%%%%%%
\section{Introduction}
\label{sec:intro}
A detailed understanding of the performance of stochastic gradient descent (SGD) in neural network is a major endeavour in machine learning, and  significant progress was achieved in the context of large two-layers neural networks. In particular the optimisation over wide two-layer neural networks can be rigorously studied using a well defined partial differential equation (PDE) \citep{mei_2018,chizat_2018,rotskoff_2019,sirignano2020mean}. A consequence of these results is the global convergence of overparametrised two-layer networks towards perfect learning provided that the number of hidden neurons is large, the learning rate is sufficiently small, and enough data is at disposal. This line of work is commonly referred to as the \emph{mean-field limit} of neural networks. The phenomenology in this regime was also studied for synthetic data with simple target functions by \cite{mei_2019,pmlr-v178-abbe22a}.

Interestingly, the SGD dynamics of two-layer neural networks trained on synthetic Gaussian data was considered as early as in the seminal work of \cite{saad_1995, saad_1995_0, saad_1996}, and has witnessed a renewal of activity over the last few years \citep{vershynin_2018_high, goldt_2019, veiga2022phase, benarous2022}. However, differently from the mean-field limit, these works investigated the opposite limit of \emph{fixed} hidden layer width and \emph{diverging} data dimension, and studied the limiting SGD dynamics through a set of ordinary differential equations (ODEs).

Given these different limits, one may naturally wonder what is the relation, if any, between these sets of works. More generally, given data in dimension $d$ and a two-layer network with $p$ hidden units trained by SGD with a learning rate $\gamma$, one might inquire about the different regimes beside the mean-field ($p\!\to\!\infty$) and high-dimensional ($d\!\to\!\infty)$ ones. This is the question investigated in this work. We consider a two-layer network trained on Gaussian data and labels given by a similar, though not necessarily identical, two-layer neural network target (hereafter also referred to as the \emph{teacher}), and investigate the one-pass stochastic gradient descent (SGD) dynamics as a function of the relevant parameters $d,p$ and $\gamma$. As summarised in Fig.\ref{fig:triangle}, we show that as long as $\sfrac{\gamma}{dp}\!\to\!0^+$, a unifying deterministic description can be provided. In particular, our description recovers all the previously studied limits (mean field, high-dimensional and the classical gradient flow regime) and builds a bridge between them in a unified formalism. Namely, our \textbf{main contributions} are:
\begin{itemize}
    \item Starting from the general approach of \cite{saad_1995, saad_1996}, we build on the non-asymptotic results \cite{veiga2022phase} and show how it allows to describe the entire phase space in Fig.\ref{fig:triangle}, that interpolates between the high-dimensional,  classical, and mean-field limits.
     \item We unveil a remarkable dimension independence  in the classical gradient-flow limit: once the initial conditions are given, the dynamics in terms of the sufficient statistics turns out to be {\it entirely} independent from the data dimension $d$.
     \item We explicitly construct the mean-field solution starting from the ODEs in the mean-field regime, bridging the hitherto different worlds of \citep{saad_1995, saad_1996} with the mean-field "hydrodynamic" approach. More precisely, for $p \to \infty$, we show how the ODEs simplify and give rise to a mean-field PDE, thanks to a decoupling of the learning dynamics that is found to remain close to a low-dimensional subspace spanned by the target principal directions.
    \item We further discuss how the SGD dynamics in the mean-field regime behave differently whether the input dimension $d$ is small or large, leading to different finite-$d$ and high-$d$ dimensionless PDEs. This generalizes the recent ``dimension-free'' results in \cite{chizat_meanfield_symmetries,pmlr-v178-abbe22a}.
    \item We provide a numerical solver for these equations. A GitHub repository with the code employed in the present work is available  on \href{https://github.com/IdePHICS/DimensionlessDynamicsSGD}{[https://github.com/IdePHICS/DimensionlessDynamicsSGD]}.  Additionally, we discuss the interesting case of quadratic activation that allows to drastically reduce the complexity of the ODEs \& PDEs.
\end{itemize}

\begin{wrapfigure}{rt}{0.5\textwidth}
  \begin{center}
    \includegraphics[width=0.47\textwidth]{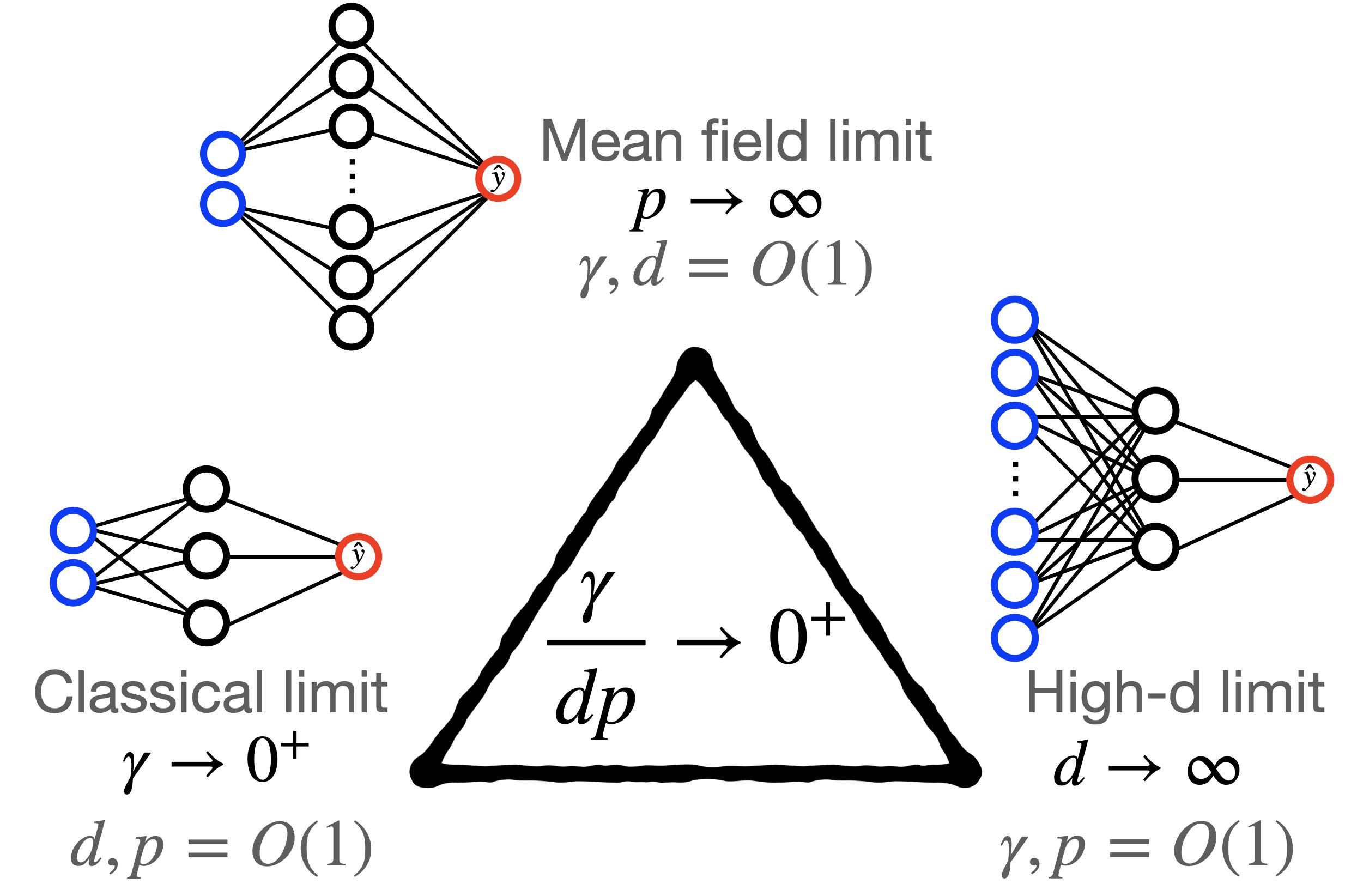}
  \end{center}
 \caption{We discuss a unifying low-dimensional description of the one-pass SGD dynamics eq.~\eqref{eq:def:sgd} for two-layers neural networks as $\sfrac{\gamma}{dp}\!\to\!0^+$. This includes, in particular, the mean-field ($p\!\to\!\infty$), the high-dimensional ($d\!\to\!\infty$) and the classical gradient flow ($\gamma\!\to\!0$) limits.}
\label{fig:triangle}
\end{wrapfigure}
\paragraph{Related work ---} Stochastic gradient descent was first introduced \cite{Robbins51} as a stochastic approximation method, and later applied as an approximation to population risk minimization in \cite{NIPS2003_9fb7b048,NIPS2007_0d3180d6}. Its properties have been extensively studied for finite learning rate and input dimension in the strongly convex setting \citep{Pflug1986, Moulines2011, Dieuleveut2017, Dieuleveut2020}, and more recently in convex problems in the interpolating regime \citep{Ma2018, Vaswani2019, Zou2021, NEURIPS2021_b4a0e0fb}, to cite a few.   High-dimensional limits of SGD were studied in \cite{vershynin_2018_high, arous2021online} for non-convex, single-index models. Recently, \cite{benarous2022} has generalised and abstracted this discussion.
    
In the context of two-layer neural networks, the high-dimensional limit of SGD draws back from the seminal work of \cite{saad_1995, saad_1995_0, saad_1996} and was subsequently studied by many authors  under different settings \citep{biehl_1995, copelli_1995, biehl_1996, goldt_2019, Goldt2020Modeling, goldt2020gaussian, refinetti2021classifying}. The infinite-width (a.k.a. mean-field) limit of the SGD dynamics of two-layer neural networks was studied by \cite{mei_2018,chizat_2018,rotskoff_2019,sirignano2020mean}, who proved global convergence under certain conditions on the architecture and initialization. A bridge between these two limits was discussed by \cite{veiga2022phase}, who studied the joint limit where the hidden-layer width and the learning rate scale with the diverging input dimension. 

Closer to us, dimension-free limits of the mean-field equations have been derived by \cite{pmlr-v178-abbe22a} for low-dimensional target functions in the hypercube and by \cite{chizat_meanfield_symmetries} for ReLU networks when the target is invariant under certain symmetries. \cite{boursier2022gradient} has proven global convergence of the gradient flow dynamics at finite width for orthogonal input data. 
%%%%%%%%%%%%%%%%%%%%%%%%%%%%%%%%%%%%%%%%%%%%%%%%%%%%%%%%%%%%%%%%%%%%%%%%%%%%%%%
\section{Setting}
\label{sec:setting}
In this manuscript we consider a supervised learning regression task where we are given $\nsamp$ independent samples $(\vec{x}^{\i},y^{\i})_{\i\in [\nsamp]}\in\mathbb{R}^{\inp+1}$ from a probability distribution $\datadist$. We are interested in the problem of learning the training data with a (fully-connected) two-layer neural network:
\begin{align}
\label{eq:def:model}
f_{\Theta}(\vec{x}) = \frac{1}{p}\sum\limits_{i=1}^{\hids}a_{i}\act(\vec{w}_{i}^{\top}\vec{x}),
\end{align}
\noindent where $\Theta = (\vec{a},\mat{W})\in\mathbb{R}^{\hids(\inp+1)}$ denote the trainable parameters and $\act:\mathbb{R}\to\mathbb{R}$ the activation function. Since one of our goals is to connect with this line of work, for convenience we have adopted the mean-field normalisation \citep{mei_2018, chizat_2018, rotskoff_2019, sirignano2020mean}. As usual, training is performed via \emph{empirical risk minimisation}, where the statistician chooses a \emph{loss function} $\ell:(f_{\Theta}(\vec{x}), y)\in\mathbb{R}^{2}\mapsto \ell(f_{\Theta}(\vec{x}), y)\in\mathbb{R}_{+}$ penalising deviations from the true labels and optimises the training parameters $\Theta$ by minimising the loss over the training data. As it is common in regression, we use the square loss $\ell(\hat{y},y) \coloneqq \sfrac{1}{2}(\hat{y}-y)^2$, and focus on the \emph{generalisation error} or \emph{population risk}:
\begin{align}
\label{eq:def:poprisk}
\mathcal{R}(\Theta) \coloneqq \mathbb{E}_{(\vec{x},y)\sim\rho}\left[\frac{1}{2}(f_{\Theta}(\vec{x})-y)^2\right]
\end{align}

\paragraph{Training algorithm:} This empirical risk minimisation problem being non-convex, different optimisation algorithms might reach different minima. Our goal is to characterise the training dynamics of two-layer networks under \emph{one-pass stochastic gradient descent}:
\begin{align}
\label{eq:def:sgd}
\Theta^{\i+1} = \Theta^{\i} -\gamma\nabla_{\Theta}\ell(f_{\Theta^{\i}}(\vec{x}^{\i}),y^{\i}), && \i \leq n
\end{align}
Note that in one-pass SGD, a fresh sample of data is used to estimate the gradient at each step, and therefore the quantity of data seen by the algorithm coincides with the number of steps. In particular, this means that at each step $\i$ we have a random unbiased estimation of the population gradient which is uncorrelated to the previous step, defining a Markov chain. It is useful to rewrite eq.~\eqref{eq:def:sgd} by making explicit the effective noise of the process:

\begin{align}
\label{eq:sgd:noise}
\Theta^{\i+1} = \Theta^{\i} -\gamma\nabla_{\Theta}\mathcal{R}(\Theta^{\i})+\gamma\varepsilon_{\i}, && \varepsilon_{\i} \coloneqq \nabla_{\Theta}\left[\mathcal{R}(\Theta^{\i}) -\ell(f_{\Theta^{\i}}(\vec{x}^{\i}),y^{\i})\right]
\end{align}
which is a zero mean random variable. Therefore, characterising the training dynamics of one-pass SGD translates into characterising this stochastic process.

\paragraph{Data model: } Stochasticity in eq.~\eqref{eq:def:sgd} is induced by the draw of samples $(\vec{x}, y)\sim \rho$. In the following, we will assume that the input $\vec{x}^\i$ are Gaussian $\vec{x}^{\i}\sim\mathcal{N}\left(\vec{0}_{d}, \sfrac{1}{d}\mat{I}_{d}\right)$ while $y^\i$ is drawn from the following generative model:
\begin{align}
\label{eq:target}
y^{\i} = \frac{1}{\hidt} \sum\limits_{r=1}^{k}a^{\star}_{r}\sigma^{\star}({\vec{w}_{r}^{\star}}^{\top}\vec{x}^{\i}) + \sqrt{\noise}z^{\i},  \qquad z^{\i}\sim\mathcal{N}(0,1)
\end{align}
In other words, the target function $f_{\Theta^{\star}}$ is itself a two-layers neural network with parameters $\Theta^{\star}=(\vec{a}^{\star}, \mat{W}^{\star})\in\mathbb{R}^{k(d+1)}$ and activation $\sigma^{\star}$. This setting, commonly refereed to as the \emph{teacher-student} scenario, provides a rich data model for studying generalisation, and has been employed both in the analysis of one-pass SGD \citep{saad_1996, goldt_2019, veiga2022phase} but also more broadly in high-dimensional statistics \citep{NEURIPS2021_9704a4fc}. In particular, we will be mostly interested in the realisable scenario where $k\leq p$, and therefore the minimum of the population risk in eq.~\eqref{eq:def:poprisk} is achieved by perfectly learning the target.

\paragraph{Technical assumptions:} We assume that the SGD dynamics remains in a bounded subset of $\mathbb{R}^{p \times d}$:
\begin{assump}\label{assump:boundedness}
    On an event with high probability, the SGD iterates are bounded in the following sense: for some $K > 0$, we have
    \begin{equation}
       \forall i \in [p], \quad\norm{\vec{w_i}}^2  \leq K.
    \end{equation}
\end{assump}
This assumption can be easily checked on either the simulations or their deterministic approximations (see below), or otherwise enforced with a weight decay (as in \cite{wang_2022_uniform}). 
\begin{assump}\label{assump:activation}
    The student activation function $\sigma$ is twice differentiable, with $\lVert \sigma^{(i)} \rVert_\infty \leq K$ for $i = 0, 1, 2$. The teacher activation $\sigma^\star$ is also upper bounded by $K$.
\end{assump}
Note that in some plots, we will sometimes use the function $\sigma(x) = x^2$, which does not satisfy this assumption. However, Assumption \ref{assump:boundedness} ensures that we stay in a bounded subset of $\mathbb R^d$, hence we can replace $\sigma$ by $\sigma \wedge K$ for $K$ sufficiently large.

\paragraph{Simplifying assumptions: } Since most of the interesting phenomenology happens at the hidden-layer, to lighten the discussion in the following we will focus on the case in which $a_{r}^{\star} = 1$ and $a^{\i}_{i}=1$ are fixed throughout learning and $p$ is divisible by $k$. All of the discussion that follows can be readily generalised to the case in which $\vec{a}^{\i}$ is learned and $p\geq k$ generically. We shall also assume that the teacher matrix $\mat{W}^\star$ is full rank; this can be avoided with a more careful definition of projections, but complicates the analysis. 
% %%%%%%%%%%%%%%%%%%%%%%%%%%%%%%%%%%%%%%%%%%%%%%%%%%%%%%%%%%%%%%%%%%%%%%%%%%%%%%%
\section{The three limit regimes and their dimensionless description}
\label{sec:regimes}
As discussed before, the optimisation problem introduced in Sec.~\ref{sec:setting} defines a non-convex optimisation problem. Moreover, modern neural networks operate in a regime where both the data dimension $d$ and the number of parameters in the network $\hids$ are large. Therefore, characterising the evolution of the weights $\Theta^{\i}\in\mathbb{R}^{\hids(d+1)}$ amounts to studying $\hids(d+1)$ non-linear, coupled, non-convex stochastic process - a challenging problem even for numerical methods. As motivated in the introduction Sec.~\ref{sec:intro}, in this section we derive a \emph{tractable}, \emph{low-dimensional} description for SGD in different regimes of practical interest. 

\subsection{Main concepts}

\paragraph{Sufficient statistics:} A first observation is that the performance of the predictor $(\Theta^{\i})_{\i\leq \nsamp}$ at iteration $\i \leq n$ only depends on the statistics of the student and teacher pre-activations 
\begin{equation}\label{eq:def:pre_activations}
    \vec{\lambda}^{\i}\coloneqq \mat{W}^{\i}\vec{x}^{\i}\in\mathbb{R}^{\hids}, \quad {\vec{\lambda}^{\star}}^{\nu}\coloneqq\mat{W}^{\star}\vec{x}^{\i}\in\mathbb{R}^{\hidt}
\end{equation}
Moreover, since $\vec{x}^\i$ is Gaussian and independent from $(\mat{W}^{\i}, \mat{W}^{\star})$, the pre-activations are jointly Gaussian vectors $(\vec{\lambda}^{\i}, {\vec{\lambda}^{\star}}^{\nu})\sim\mathcal{N}(\vec{0}_{\hids+\hidt}, \Omega^{\i})$ with covariance: 
\begin{equation}\label{eq:def:overlaps}
\begin{split}
\Omega^{\i}\coloneqq
\begin{pmatrix}
\mat{Q}^{\i}& \mat{M}^{\i}\\
{\mat{M}^{\i\top}} & \P
\end{pmatrix}=
\begin{pmatrix}
\sfrac{1}{d}\mat{W}^{\i}{\mat{W}^{\i}}^{\top}& \sfrac{1}{d}\mat{W}^{\i}{\mat{W}^{\star}}^{\top}\\
\sfrac{1}{d}\mat{W}^{\star}{\mat{W}^{\i}}^{\top}& \sfrac{1}{d}\mat{W}^{\star}{\mat{W}^{\star}}^{\top}
\end{pmatrix}
\in\mathbb{R}^{(\hids+\hidt)\times (\hids+\hidt)}
\end{split}
\end{equation}
which is the \emph{sufficient statistics} matrix for the population risk in eq.~\eqref{eq:def:poprisk}. Massaging eq.~\eqref{eq:def:sgd}, we can derive a closed set of stochastic processes governing the evolution of the sufficient statistics:
\begin{equation}
\label{eq:def:overlap_process}
\begin{split}
M^{\i+1}_{ir} - M^{\i}_{ir} &= \frac{\gamma}{\hids d}\sigma'(\lambda_i^\i)\lambda^{\star\i}_{r}\, \mathcal{E}^{\i}\\
Q_{ij}^{\i+1} - Q^{\i}_{ij} &=  \frac{\gamma}{\hids d}\left(\sigma'(\lambda_i^\i)\lambda^{\i}_{j}+\sigma'(\lambda_j^\i)\lambda^{\i}_{i}\right) \mathcal{E}^\i+\frac{\gamma^2 ~||\vec{x}^{\i}||_{2}^{2}}{p^2d^2}	\sigma'(\lambda_i^\i)\sigma'(\lambda_j^\i)\,{\mathcal{E}^{\i}}^2
\end{split}
\end{equation}
where we defined for convenience the displacement vector 
\begin{equation}
\mathcal{E}^{\i} \coloneqq \frac{1}{k}\sum_{r=1}^{\hidt}\sigma^{\star}(\lambda_{r}^{\star\i})-\frac{1}{p}\sum_{j=1}^{\hids}\sigma(\lambda_{j}^{\nu})+\sqrt{\noise} z^{\i}.
\end{equation}
Note that so far we have made no approximations: these equations are exact, and allow us to trade the $\hids(d+1)$ dimensional process for $\Theta^{\i}$ in Eq.~\eqref{eq:def:sgd} for a $\hids(\hidt+\hids)$ dimensional process for $(\mat{M}^{\i},\mat{Q}^{\i})$. This can be particularly convenient if $d\gg \hidt, \hids$. 

The process in Eq.~\eqref{eq:def:overlap_process} has been previously studied in different limits and particular cases. For instance, \cite{tan2019b} studied the stochastic dynamics in particular case of $\hidt=\hids=1$ and $\sigma(x)=x^2$, also known as \emph{phase retrieval}, and \cite{arous2021online} extended this discussion for arbitrary $\sigma$. More important to our work, \cite{saad_1995, saad_1996} has shown that this process admits a deterministic limit when $d\!\to\!\infty$ at fixed $p,\gamma$, characterized by the following ODE:
\begin{align}
\tag{SS-ODE}\label{eq:def:ss_ode}
\frac{\dd \mat{M}}{dt} = \Psi^{(\mathrm{M})}(\Omega), &&
\frac{\dd\mat{Q}}{\dd t} =  \Psi^{(\mathrm{GF})}(\Omega) +\frac{\gamma}{p} \Psi^{(\mathrm{Var})}(\Omega),
\end{align}
where the right-hand side functions are defined by the following equations.
\begin{equation} \label{eq:def:Psi}
\begin{split}
\Psi^{(\mathrm{M})}_{ir}(\Omega) &= \mathbb E_{(\vec{\lambda},\vec{\lambda}^{\star})\sim\mathcal{N}(\vec{0}_{\hids+\hidt}, \Omega)} \left[\sigma'(\lambda_i)\lambda^{\star}_{r}\,\mathcal{E}\right]\\
\Psi^{(\mathrm{GF})}_{ij}(\Omega) &=  \mathbb E_{(\vec{\lambda},\vec{\lambda}^{\star})\sim\mathcal{N}(\vec{0}_{\hids+\hidt}, \Omega)} \left[\left(\sigma'(\lambda_i)\lambda_{j}+\sigma'(\lambda_j)\lambda_{i}\right) \mathcal{E}\right] \\
\Psi^{(\mathrm{Var})}_{ij}(\Omega) &= \mathbb E_{(\vec{\lambda},\vec{\lambda}^{\star})\sim\mathcal{N}(\vec{0}_{\hids+\hidt}, \Omega)} \left[\sigma'(\lambda_i)\sigma'(\lambda_j)\,{\mathcal{E}}^2 \right]
\end{split}
\end{equation}
As will be discussed later in Section \ref{sec:classic}, the superscript notation for the right-hand side is suggestive of their interpretation. This convergence was made rigorous by \cite{goldt_2019} and \cite{veiga2022phase}, who showed the following non-asymptotic result:

\begin{theorem} [\cite{veiga2022phase}] \label{th:conv_eps} We place ourselves under Assumptions \ref{assump:boundedness}-\ref{assump:activation}. Let $\Omega^\i$ be the random process of Eq. \eqref{eq:def:overlap_process}, and $\Omega(t)$ the solution to the ODE \eqref{eq:def:ss_ode} with starting point $\Omega(0) = \Omega^0$. Define the stepsize $\delta t = \sfrac{\gamma}{pd}$, and assume that $\gamma / p = O(1)$.
Then there exists a constant $C > 0$ such that for any $\i \geq 0$,
\begin{equation}
    \lVert \Omega^{\nu} - \Omega(\nu\delta t) \rVert_\infty \leq e^{C\i \delta t} \sqrt{\frac{\gamma}{pd}}
\end{equation}
\end{theorem}

\cite{veiga2022phase} also described the behavior of $\Omega$ for various choices of $\gamma$ and $p$. In particular, when $\gamma / p \ll 1$, equations \eqref{eq:def:ss_ode} reduce to the following simpler ones:
\begin{align}
\tag{GF-ODE}\label{eq:def:gf_ode}
\frac{\dd\mat{M}}{\dd t} = \Psi^{(\mathrm{M})}(\Omega), && 
\frac{\dd\mat{Q}}{\dd t} =  \Psi^{(\mathrm{GF})}(\Omega)
\end{align}
Our key observation is that the deterministic description above is valid beyond the high-dimensional limit $d\to\infty$ on which previous works \citep{saad_1995,goldt_2019, veiga2022phase} have focused. Indeed, Thm.~\ref{th:conv_eps} is non-asymptotic in $(d,\hids,\gamma)$, and can thus be applied in any setting where $\sfrac{\gamma}{pd} \ll 1$. This is leveraged to provide a tractable, low-dimensional description of SGD in different scenarios of interest which we summarise in Fig.~\ref{fig:triangle}. 

%%%%%%%%%%%%%%%%%%%%%%%%%%%%%%%%%%%%%%%%%%%%
\subsection{The classical regime}
\label{sec:classic}
%%%%%%%%%%%%%%%%%%%%%%%%%%%%%%%%%%%%%%%%%%%%
The first and most well-studied scenario is the classical regime in which $\gamma\to 0^{+}$ at fixed dimensions $d, \hids = O(1)$.  Defining the continuous weight $\Theta^{\i} = \Theta(\nu\delta t)$ via linear interpolation, a classical result from stochastic optimisation  \citep{Robbins51} is that one-pass SGD converges to gradient flow on the population risk:
\begin{align}
\label{eq:gf}
\frac{\dd \Theta(t)}{\dd t} = -\nabla_{\Theta}\mathcal{R}(\Theta(t)) 
\end{align}
Or in words: the effective SGD noise $\varepsilon_{\nu}$ in eq.~\eqref{eq:sgd:noise} is subleading in this limit. Note that this is a deterministic ordinary differential equation of dimension $\hids(d+1)$. Since $d,\hids = O(1)$, if they are small eq.~\eqref{eq:gf} provides a computationally efficient description of the SGD dynamics, since it can be easily implemented and solved in a computer. Nonetheless, an alternative description can be derived from the sufficient statistics of eq.~\eqref{eq:def:overlaps}. Indeed, Thm.~\ref{th:conv_eps} guarantees that in the limit $\gamma\to 0^{+}$, the stochastic process in eq.~\eqref{eq:def:overlap_process} converges to the deterministic limit of eq. \eqref{eq:def:gf_ode}. This ODE can be easily seen to be equivalent to the one of \eqref{eq:gf}, through the following identity:
\[ \frac{\dd\, \langle \vec{w_i}, \vec{w_j} \rangle}{\dd{t}} = \left\langle \vec{w_i}, \frac{\dd\vec{w_j}}{\dd{t}} \right\rangle + \left\langle \frac{\dd\vec{w_i}}{\dd{t}}, \vec{w_j} \right\rangle \]
This gives rise to ordinary differential equations in $\hids (\hidt+\hids)$ parameters. Therefore, depending on the values of $(d,\hids)$, it can offer a more compact description of the evolution of the performance of the predictor than eq.~\eqref{eq:gf}. 

As it was previously hinted by the notation, equation \eqref{eq:def:gf_ode} also provides an intuitive interpretation of the right-hand side of the stochastic process \eqref{eq:def:overlap_process}. Indeed, the terms proportional to $\sfrac{\gamma}{\hids d}$ in eq.~\eqref{eq:def:overlap_process} correspond exactly to the terms inside of the expectation in eq.~\eqref{eq:def:gf_ode}, and correspond to the projection of the population gradient along the weights $(\Theta, \Theta_{\star})$. The remaining term, which is proportional to $\sfrac{\gamma^2}{\hids^2 d}$, comes from the variance of the effective noise $\varepsilon$, which is subleading in the limit $\gamma\to 0^{+}$. This agrees with the characterization of the terms given in \cite{benarous2022}, where an additional Brownian motion correction term at finer scales is also derived.

In Figure~\ref{fig:classic-limit}~(left) we plot the trajectories of individual neurons in the space spanned by the two target neurons (\(k=2\)); \(\sfrac{\M_{jr}}{\sqrt{\Q_{jj}\P_{rr}}}\) is the (normalised) scalar product between \(\w_j\) and \(\w^*_r\). While, initially, all neurons are pointing to the superposition of the two teacher weights (\((\sfrac{\sqrt2}{2},\sfrac{\sqrt2}{2})\), yellow dot), in the last phase of learning they "specialize" and split evenly to one of the two (\((1,0),(0,1)\), red squares). Note how the ODE trajectories follows closely the simulated ones.

Finally, our results imply a remarkable {\bf dimension independence} property: Differently from \eqref{eq:gf}, the alternative description \eqref{eq:def:gf_ode} turns out to be independent of the data dimension $d$. Given the initial conditions of the sufficient statistics (the overlaps), then the trajectories will behave {\bf exactly in the same way} whether $d$ is large or small (see the illustration in Figure~\ref{fig:classic-limit}~(right)). This remarkable property is a direct consequence of the Gaussianity assumption on the data. Note, of course, that dimensionality still plays a crucial role through the initialisation. Indeed, for the typically employed random initialization $\theta_{i}^{0}\sim\mathcal{N}(0_{d},\sigma^2\mat{I}_{d})$, the initial correlation between the hidden-units and the target, parameterised by $M^{0}_{rj}\sim \sfrac{1}{\sqrt{d}}$, explicitly depends on $d$. Since $M\!=\!0$ is often a fixed-point of the dynamics, $d\!=\!O(1)$ or $d\!\gg\!1$ will lead effectively to different behaviours. 

\begin{figure}[t]
    \centering
        \includegraphics[width=0.44\textwidth]{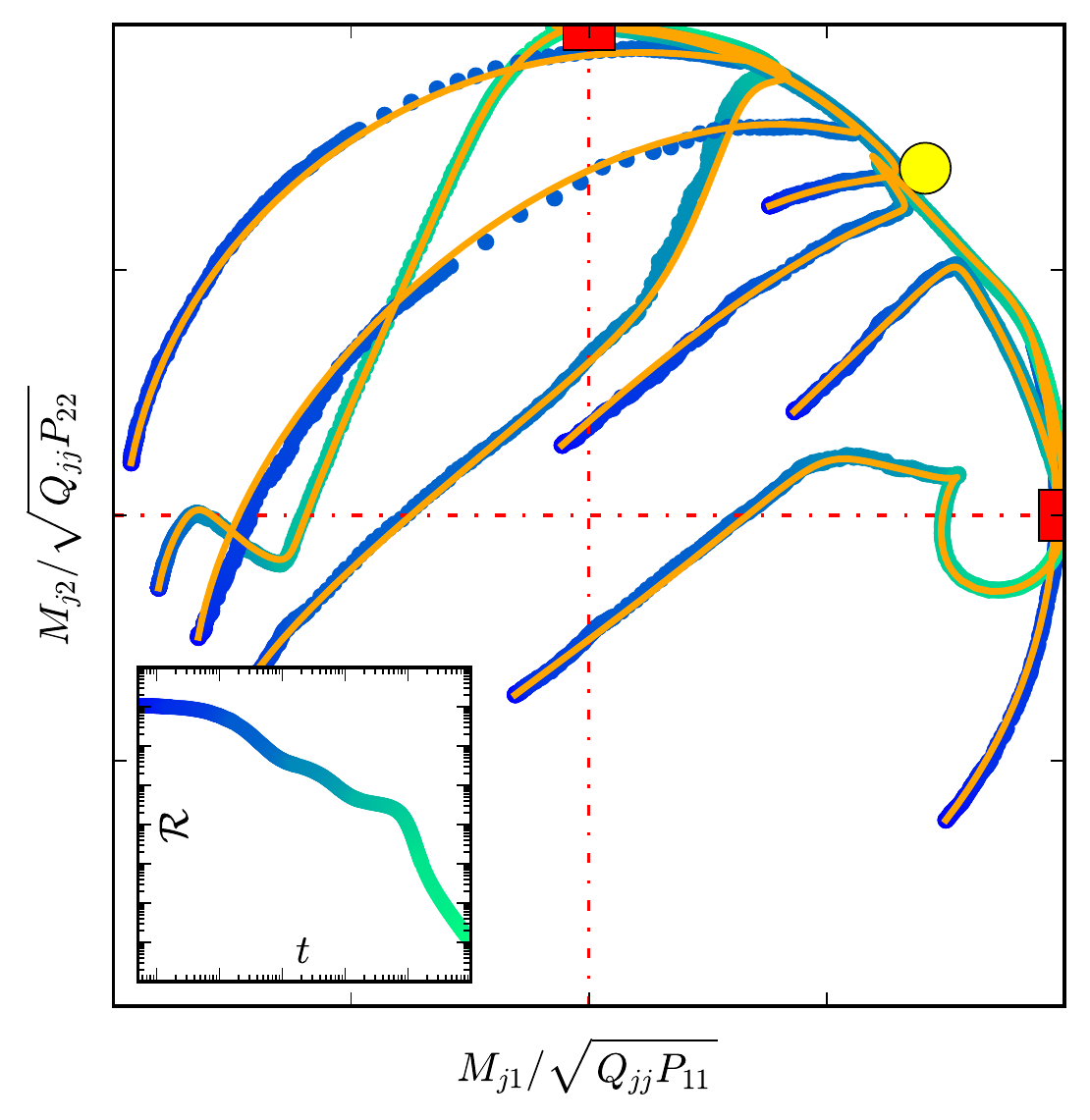}
        \includegraphics[width=0.49\textwidth]{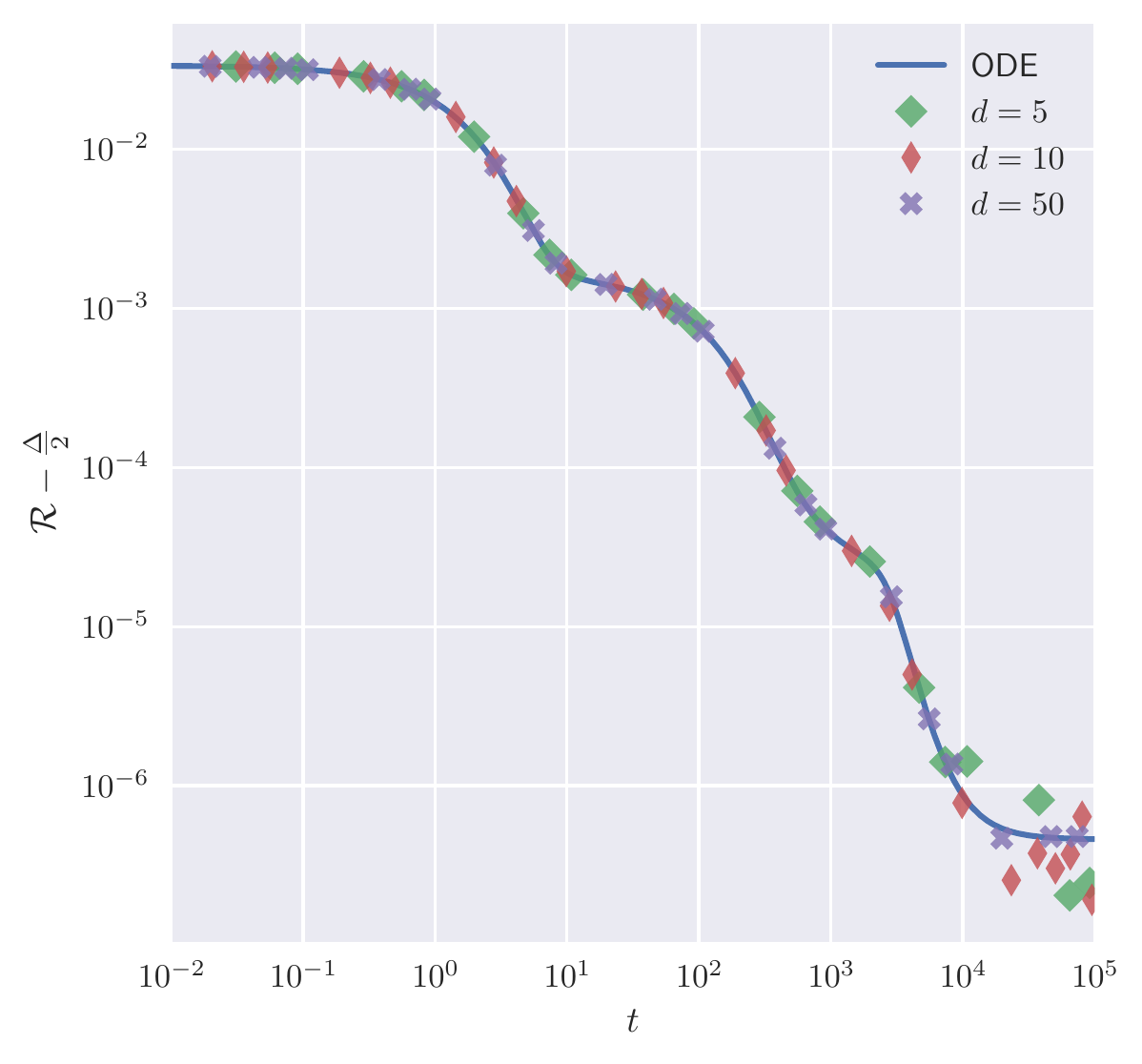}
    \caption{{\bf Dimension independence of the dynamics in the classical regime}:
Comparison between the simulated SGD  dynamic and the analytical one obtained by integrating the differential equations. All the plots are using $\sigma=\erf(\sfrac{\cdot}{\sqrt{2}})$. (Left:) Trajectories for the cosine similarity between each network neural and the target; time evolution in simulation is blue to green, ODE trajectories in orange (\(\hidt=2,\hids=10,\lr=\num{5e-2},\noise=\num{0.}\)). (Right:)  Different low-dimensional simulations, starting from the same identical initial condition (\(\hidt=5,\hids=10,\lr=\num{5e-2},\noise=10^{-3}            \)).
    }
    \label{fig:classic-limit}
\end{figure}

%%%%%%%%%%%%%%%%%%%%%%%%%%%%%%%%%%%%%%%%%%%%
\subsection{The high-dimensional regime}
\label{sec:saadsolla}
%%%%%%%%%%%%%%%%%%%%%%%%%%%%%%%%%%%%%%%%%%%%
Modern machine learning practice often involves high-dimensional data. As early as in
\citep{saad_1995, saad_1996} it has motivated the study of the $d\to\infty$ limit of the SGD dynamics \eqref{eq:sgd:noise}, under the assumption of fixed learning rate and model complexity $p,\gamma = O(1)$. This setting has witnessed a renewal of interest recently \cite{vershynin_2018_high,goldt_2019,veiga2022phase,benarous2022}. In particular, a remarkable phenomenon arises: differently from the classical limit eq.~\eqref{eq:def:gf_ode} discussed above, in high-dimension the variance term induced by the SGD effective noise yields an explicit contribution to the limiting dynamics \eqref{eq:def:ss_ode}. This term can yield a finite risk contribution at large times even for architectures for which the population gradient flow would otherwise converge to zero population risk  (i.e. perfect learning of the target $f_{\Theta^{\star}}$), so that the large time dynamics plateau at finite risk (see for instance \cite{goldt_2019,benarous2022}). This is a major difference between the classical and high-dimensional regime.

For strongly convex problems, it is known that SGD with fixed learning rate $\gamma$ converges to a stationary distribution of variance $\propto \gamma$ \cite{Pflug1986, Dieuleveut2020}, leading to an asymptotic risk that closely resembles the one observed by \cite{saad_1995} in the high-dimensional regime. This suggests a similar phenomenology in the basin of the global minima, although making this statement precise is challenging due to the non-convexity of the risk for two-layer networks.  As noted by \cite{veiga2022phase}, this noise term is subleading in $\sfrac{\gamma}{p}$, and can be mitigated by either taking $\gamma\to 0^{+}$ (i.e. seeing a lot of data) or overparametrising $\hids\to\infty$. However, since eqs.~\eqref{eq:def:ss_ode} are a system of $\hids(\hids+\hidt)$ ordinary differential equations, they become intractable in the limit $\hids\to\infty$, which is a major shortcoming of this description. 

%%%%%%%%%%%%%%%%%%%%%%%%%%%%%%%%%%%%%%%%%%%%
\subsection{The overparametrised regime}
\label{sec:meanfield}
%%%%%%%%%%%%%%%%%%%%%%%%%%%%%%%%%%%%%%%%%%%%
In both the classical and high-dimensional regimes, the effective description of the SGD dynamics rely on quantities which scale with the hidden-layer width $\hids$, and therefore they are not adequate to wide models. Yet, in many scenarios of interest we need to deal with wide, overparametrised networks. The problem is finding an effective low-dimensional description of one-pass SGD for the overparametrised regime $\hids\to\infty$ was first addressed by \cite{mei_2018, chizat_2018, rotskoff_2019, sirignano2020mean}. The key idea in this line of work is to define an empirical density over the weights $\theta^{\i}_{i}\coloneqq (a_{i}^{\i},\vec{w}^{\i}_{i})$:
\begin{align}
\hat{\mu}^{\i}_{\hids}(\vec{\theta}) = \frac{1}{p}\sum\limits_{i=1}^{\hids} \delta(\vec{\theta}-\vec{\theta}_{i}^{\i})
\end{align}
\noindent and to derive a closed-form update for the density $\hat{\mu}^{\i}_{\hids}(\vec{\theta})$ from the SGD update of the weights, eq.~\eqref{eq:sgd:noise}. In the limit $\hids\to\infty$, those works have shown that the empirical density converges to an asymptotic density over $\mathbb{R}^{d+1}$, which for sufficiently small learning rate satisfies a partial differential equation (PDE) that became known in the literature as the \emph{mean-field limit}. Drawing from the theory of PDEs and optimal transport, this description allowed for the derivation of important mathematical guarantees on the dynamics, such as the global convergence of SGD for two-layers neural networks. 

However, the empirical measure $\hat \mu_p^\nu$ is defined on $\mathbb R^{d+1}$, so a problem still remains when $d$ is large. Indeed, as remarked in \cite{bach2021gradient} it remains challenging to draw quantitative results from this description except for considerably low-dimensional data. However, since in our setting the target function \eqref{eq:target} only acts on a low-dimensional subspace of $\mathbb R^d$, it is possible to exploit the symmetries of the problem to derive an approximation of constant dimension. This low-dimensional equivalent stems from invariance properties of mean-field equations, which were also used in \cite{pmlr-v178-abbe22a} and studied in depth in \cite{chizat_meanfield_symmetries}. However,  \cite{pmlr-v178-abbe22a} only considers the $d \to \infty$ limit, while we derive a limit that is valid for any value of $d$.  \cite{chizat_meanfield_symmetries}, on the other hand, is closer to our work (see e.g. their Lemma 4.2), but only handles the approximation of the dynamics by PDEs instead of ODEs.

\paragraph{Decomposing the dynamics: } The starting point of this low-dimensional description is the decomposition of $W$ as
\begin{equation}\label{eq:W_proj}
    W = W_{\text{proj}} + W^\bot,
\end{equation}
where $W_{\text{proj}}$ is the orthogonal projection on the teacher vectors $W^\star$. This projection can be expressed using the sufficient statistics defined in eq.~\eqref{eq:def:overlaps}:
\begin{equation}\label{eq:W_proj_simplified}
    W =  M P^{-1} W^\star + W^\bot.
\end{equation}
Similar to \eqref{eq:def:pre_activations}, we can then define the orthogonal pre-activations and and its covariance matrix:
\begin{align}\label{eq:def:lambda_orth}
    \vec{\lambda}^{\bot\i} &= W^{\bot\i}\vec{x}^\i \\
    \label{eq:def:Qorth}
    Q^\bot &= W^\bot (W^\bot)^\top = Q - MP^{-1}M^\top\, .
\end{align}
Since we are in a regime where $\gamma / p \ll 1$, the ODE approximation of SGD corresponds to equations \eqref{eq:def:gf_ode}. From there, with a little algebra (see Appendix \ref{app:mean_field_orthogonal_derivation}), we can derive the corresponding equations for $Q^\bot$:
\begin{equation} \label{eq:def:mf_orthogonal}
    \frac{dQ_{ij}^\bot}{dt} =  \mathbb E_{(\vec{\lambda}^{\bot},\vec{\lambda})}\left[\left( \sigma'(\lambda_i)\lambda_j^\bot + \sigma'(\lambda_j)\lambda_i^\bot \right) \mathcal{E} \right] \coloneqq \Psi^\bot_{ij}(\Omega).
\end{equation}

\paragraph{Low-dimensional approximation:} Informally, the interesting part of the dynamics happens in a low-dimensional space: the one spanned by the target weights $W^\star$. The remainder of the dynamics only depends on the student-student vector interactions, which are orthogonally invariant. We therefore make the following assumption to enforce this invariance at the start:
\begin{assump}\label{assump:init}
    The initial vectors $(\vec{w}_1, \dots, \vec{w}_p)$ are drawn i.i.d from an orthogonally invariant and $\sfrac{K^2}{d}$-subgaussian distribution $\mu_0$, for some constant $K > 0$.
\end{assump}

We show in the appendix how we can then approximate the dynamics by only tracking the evolution of $M$ and $\mathrm{diag}(Q^\bot)$, and using the following ansatz:
\begin{equation}
\label{eq:mf_approx_lowd}
    \vec{w}^\bot_i \approx \sqrt{q_{ii}^\bot} \cdot \vec{g}_i,
\end{equation}
where $\vec{g}_i$ are i.i.d uniform random variables on $\mathcal S^{d-k-1}$ independent from the $q_{ii}^\bot$. More precisely, we consider the reduced parameters $\tilde\Theta = (M, q) \in \mathbb R^{p(k+1)}$, and the following mean-field equivalent of the  overlaps:
\begin{equation}
\label{eq:def:overlaps_mf}
\tilde\Omega = \begin{pmatrix} \tilde Q & M \\[0.2em]
    M^\top & P 
\end{pmatrix}, \quad \tilde{Q} = MP^{-1}M^\top +   D_{\sqrt{q}} \Xi D_{\sqrt{q}}
\end{equation}
where $D_{\sqrt{q}}$ is the diagonal matrix whose entries are the $\sqrt{q_{i}}$, and $\Xi$ is a \emph{random} matrix with independent entries such that
\[\Xi_{ii} = 1, \quad \Xi_{ij} = \langle \vec{g}, \vec{g'} \rangle \quad \text{with}\quad  \vec{g}, \vec{g'} \sim \mathrm{Unif}\left(\mathbb S^{d-k-1}\right)\]
Then, the mean-field ODEs read
\begin{equation}
\tag{MF-ODE}\label{eq:def:mf_ode}
    \frac{\dd M}{\dd t} = \mathbb{E}_{\Xi}\left[\Psi^{(M)}(\tilde\Omega)\right] \quad \frac{\dd q_i}{\dd t} = \mathbb{E}_{\Xi}\left[\Psi^{\bot}_{ii}(\tilde\Omega)\right].
\end{equation}
Similarly, given the parameters $\Theta$, the risk is computed as
\begin{equation}\label{eq:def:mf_risk}
    \mathcal{R}(\tilde\Theta) = \mathbb{E}_{\Xi}\left[\mathcal{R}(\tilde\Omega)\right]
\end{equation}
The consistency of this approximation is given by the following theorem:
\begin{theorem}\label{thm:mf_approx_lowd_bound}
 Under asms. \ref{assump:activation} and \ref{assump:init}, let $\mat{\Omega}(t)$ and $\tilde{\mat{\Theta}}(t)$ denote the solutions of the ODES \eqref{eq:def:gf_ode} and \eqref{eq:def:mf_ode}, respectively. Then with probability at least $1 - e^{-z^2}$ on the initialization:
    \[ \sup_{t \in [0, T]} \left| \mathcal R(\mat{\Omega}(t)) - \mathcal R(\tilde{\mat{\Theta}}(t)) \right| \leq {Ce^{CT} \left(\sqrt{\log(pT)} + z \right)}/{\sqrt{p}}\,.\]
\end{theorem}
The proof is given in Appendix \ref{sec:app:mf}. It uses key elements from the mean-field study of \cite{mei_2019}. Indeed, both the solutions of \eqref{eq:def:gf_ode} \& \eqref{eq:def:mf_ode} can be viewed as the ``particle dynamics'' approximations (see e.g. \cite{chertock_2017_chapter}) of two mean-field PDEs $\mu_t, \tilde\mu_t$ on the space of network weights $\vec{w} \in \mathbb{R}^{d}$ and the space or reduced parameters $(\vec{m}, q) \in \mathbb{R}^{k+1}$, respectively. In turn, the invariance property of $\mu_0$ extends to $\hat\mu_t$ for every $t \geq 0$, which implies that $\mathcal R(\mu_t) = \mathcal R(\tilde\mu_t)$. 

\begin{remark}
    We do not imply in any way, shape or form that Equation \eqref{eq:mf_approx_lowd} represents the actual distribution of the $\vec{w}_i^\bot$, or that (as in \eqref{eq:def:overlaps_mf}) then entries of $Q^\bot$ are independent. Rather, it is the structure of the update equations that allows for this approximation to hold.
\end{remark}
\begin{figure}[t]
    \centering
        \includegraphics[width=0.49\textwidth]{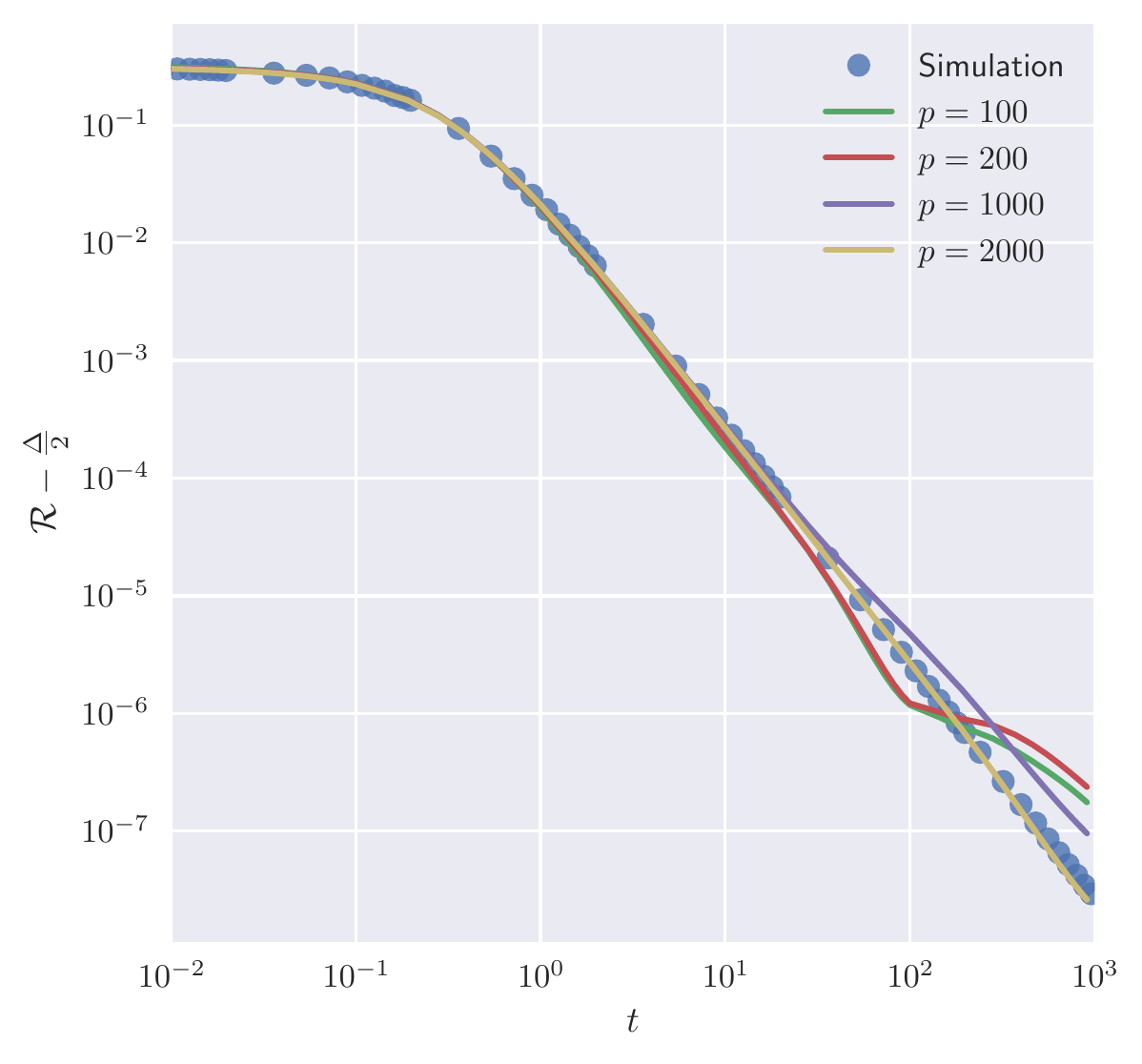}
        \includegraphics[width=0.49\textwidth]{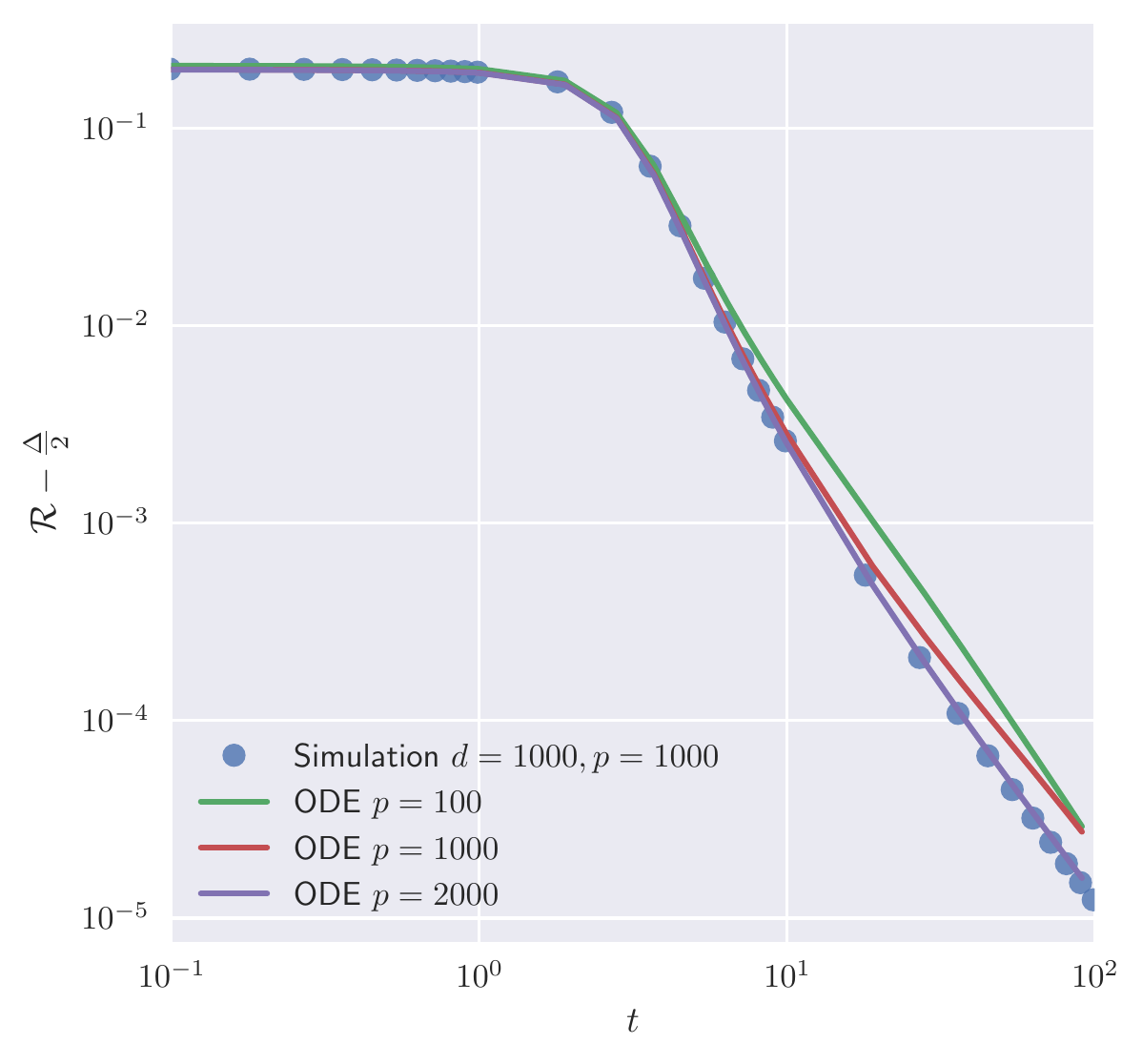}
    \caption{{\bf The mean-field regime:}    Comparison between population risks of the simulated learning dynamic and the corresponding deterministic evolution, for (left) mean-field low-dimensional regime with squared activation function (\(\hidt=2,d=5,\lr=\num{1.},\noise=\num{0.}\)) and (right:) mean-field high-dimensional regime. Apart from finite size effects, there is agreement between ODEs and simulated dynamics (\(\hidt=5,d=1000,\lr=\num{10.}\)).
    }
    \label{fig:meanfield-limit}
\end{figure}
\paragraph{The high-dimensional limit of mean-field: } When $d \to \infty$, the off-diagonal entries of the matrix $\Xi$ are of order $\sfrac{1}{\sqrt{d}}$, which suggests that they can be neglected safely. We therefore define the following equivalent of $\Omega_{\mathrm{MF}}$, which does not depend on auxiliary random variables:
\begin{equation}
    \label{eq:def:overlaps_mf_highdim}
    \bar\Omega = \begin{pmatrix} \bar Q & M \\[0.2em]
        M^\top & P 
    \end{pmatrix}, \quad \bar Q = MP^{-1}M^\top + \mathrm{diag}(q)
\end{equation}
The following lemma then holds:
\begin{lemma}\label{lem:Psi_lipschitz}
    For any $(M, q)\in \mathbb{R}^{p(k+1)}$, we have:
    \begin{equation}
    \left\lVert\mathbb{E}_{\Xi}\left[\Psi^{(M)}(\tilde\Omega)\right] - \Psi^{(M)}(\bar\Omega)\right\rVert_\infty \leq \frac{C}{\sqrt{d}}\,,
    \text{and}\, 
     \left\lVert\mathbb{E}_{\Xi}\left[\Psi^{(\bot)}(\tilde\Omega)\right] - \Psi^{(\bot)}(\bar\Omega)\right\rVert_\infty \leq \frac{C}{\sqrt{d}}\,.
\end{equation}
\end{lemma}
We thus define the high-dimensional equivalent of \eqref{eq:def:mf_ode}:
\begin{equation}
    \tag{HDMF-ODE}\label{eq:def:ode_mf_highdim}
    \frac{\dd{M}}{\dd{t}} = \Psi^{(M)}(\bar \Omega) \qquad \frac{\dd{q}_i}{\dd{t}} = \Psi^{\bot}_{ii}(\bar \Omega).
\end{equation}
The same propagation of perturbations arguments discussed in Appendix B of \cite{veiga2022phase}, and the Lipschitz property of the risk, easily imply the following theorem:
\begin{theorem}
    Under Assumptions \ref{assump:activation} \& \ref{assump:init}, let $\Omega(t)$ and $\bar{\Theta}(t)$ denote the solutions of the ODES \eqref{eq:def:gf_ode} and \eqref{eq:def:ode_mf_highdim}, respectively. Then with probability at least $1 - e^{-z^2}$ on the initialization:
    \[ \sup_{t \in [0, T]} \left| \mathcal R(\Omega(t)) - \mathcal R(\bar{\Theta}(t)) \right| \leq Ce^{CT}\left(\frac{\sqrt{\log(pT)} + z}{\sqrt{p}} + \frac1{\sqrt{d}} \right) \]
\end{theorem}
This approximation eliminates the need to compute expectations in \eqref{eq:def:mf_ode}. We now argue that this phenomenon is not only a consequence of rotation invariance, but also a simple concentration result. Indeed, irrespective of Assumption \ref{assump:init}, we show the following result:
\begin{equation}\label{eq:mf:concentration}
    \left|\mathbb E_{\vec{\lambda^\star}, \vec{\lambda^\bot} \sim \mathcal{N}(\vec{0}, \Omega)}\left[\sigma(\lambda_i)\lambda_i^\bot \mathcal E \right] - \mathbb E_{\vec{\lambda^\star}, \vec{\lambda^\bot} \sim \mathcal{N}(\vec{0}, \bar\Omega)}\left[\sigma(\lambda_i)\lambda_i^\bot \mathcal E \right]\right| \leq c\,\sqrt{Q^\bot_{ii} \cdot \frac{\lVert Q^\bot \rVert_{\mathrm{op}}}{p}},
\end{equation}
For a random sub-gaussian initialization, classical matrix concentration arguments (see \cite{vershynin_2018_high}, Theorem 4.4.5) imply
\[ \lVert Q^\bot \rVert_{\mathrm{op}} = O\left(1 + \frac pd \right) \]
and common sense arguments (the presence of an attracting force towards $Q^\bot = 0$) indicates that this quantity stays within the same order of magnitude. On the other hand, by Assumption \ref{assump:boundedness}, the $Q_{ii}$ (and hence the $Q_{ii}^\bot$) remain bounded by an absolute constant during the trajectory. We therefore expect, using standard results from ODE perturbation, that

\begin{equation}
    \lVert \bar\Omega(t) - \tilde\Omega(t) \rVert_\infty \leq e^{Ct}\left(\sqrt{\frac1p}+\sqrt{\frac1d}\right).
\end{equation}
Contrary to the low-dimensional regime and the theorem above, this derivation does not require any rotation invariance of the $\vec{w}_i^\bot$; simply, the averaging properties of $\Psi^{(M)}$ are enough to show direct concentration properties on the dynamics.

It is instructive to look at a concrete example where the phenomenology discussed above becomes explicit. Perhaps the simplest one  is given by the square activation $\sigma(x)=x^2$, for which the expectation in eq.~\eqref{eq:def:Psi} can be explicitly expressed in terms of polynomials of the covariance matrices $(\mat{Q}^{\nu},\mat{M}^{\nu}, \mat{P})$ (see Appendix \ref{sec:app:squared} for the derivation):
\begin{equation}
\label{eq:squared_ODE}
\begin{split}
    \Psi^{(\mathrm{M})}(\Omega)   &= 2\left(\frac{\Tr{\P}}k - \frac{\Tr{\Q}}p\right)\M +
        4\left(\frac{\P\M}{k} - \frac{\M\Q}{p}\right) \\
    \Psi^{(\mathrm{GF})}(\Omega)  &= 4\left(\frac{\Tr{\P}}k - \frac{\Tr{\Q}}p\right)\Q +
        8\left(\frac{\M^\top\M}{k} - \frac{\Q^2}{p}\right) \\
    \Psi^{(\mathrm{\bot})}(\Omega) &= 4\left(\frac{\Tr{\P}}k-\frac{\Tr{\Q}}{p}\right) \Q^\bot - \frac{4}{p}\left(\Q\Q^\bot+\Q^\bot\Q\right).
\end{split}
\end{equation}
Similarly, the population risk as a function of the sufficient statistics reads:
\begin{equation} \label{eq:risk_quadratic}\begin{split}
  \mathcal{R}(\Omega)
     \! =\! \frac{\Tr{\P}^2+2\Tr{\P^2}}{2k^2}
          \!-\!\frac{\Tr{\P}\Tr{\Q}+2\Tr{\M\M^\top}}{pk}
          \!+\!\frac{\Tr{\Q}^2+2\Tr{\Q^2}}{2p^2}
          \!+\!\frac\Delta2.
\end{split}\end{equation}

The high-dimensional mean field limit is then particularly simple. Recalling the definition of \(\bar\Omega\) in Equation~\eqref{eq:def:overlaps_mf_highdim}, we obtain  the explicit equations by replacing \(\Omega\) with \(\bar\Omega\). The situation is, however, different for the low-dimensional mean-field limit, due to the randomness introduced by matrix \(\Xi\). The 
ODEs and the risk, given the parameters $\tilde\Theta$, are
\begin{equation}\label{eq:def:mf_square}\begin{split}
    \mathbb{E}_{\Xi}\left[\Psi^{(M)}(\tilde\Omega)\right] &= \Psi^{(M)}(\bar\Omega) \qquad
    \mathbb{E}_{\Xi}\left[\Psi^{(\mathrm{noise})}(\tilde\Omega)\right] = \Psi^{(\mathrm{noise})}(\bar\Omega) \\
    \mathbb{E}_{\Xi}\left[\Psi^{\bot}_{ii}(\tilde\Omega)\right] &= \Psi^{\bot}_{ii}(\bar\Omega) - \frac8p \frac{\sum_{j=1,j\neq i}^\hids q_j}{d-k}q_i \\
    \mathbb{E}_{\Xi}\left[\mathcal{R}(\tilde\Omega)\right] &= \mathcal{R}(\bar\Omega) + \frac{1}{p^2} \frac{\sum_{i=1}^\hids\sum_{j=1,j\neq i}^\hids q_iq_j}{d-k}.
\end{split}\end{equation}
In this case, the low-dimensional corrections are therefore just additive terms. As expected, these corrections vanish when \(d\to\infty\), where we fall back to the high-dimensional mean-field. Conversely, when \(d= k\) the correction diverges, but we don't need to track \(q\) anymore since the teacher weights are spanning the whole space \(\mathbb{R}^d\), and the orthogonal space is null; hence \(\Q^\bot = 0\) and this is a stable point of the dynamics. In Figure~\ref{fig:meanfield-limit} we show some numerical experiments using this activation function; for a more detailed discussion see Appendix~\ref{sec:app:numerics-mean-field}.
% %%%%%%%%%%%%%%%%%%%%%%%%%%%%%%%%%%%%%%%%%%%%%%%%%%%%%%%%%%%%%%%%%%%%%%%%%%%%%%%
\section{Conclusion}
\label{sec:conclusion}
Our work provides a comprehensive analysis of the one-pass SGD dynamics of two-layer neural networks. The study bridges different regimes of interest, offers a unifying picture of the limiting SGD dynamics, sheds light on the behavior of neural networks trained on synthetic data and, we believe, provides a useful tool for further investigations of the performance of these networks.
% %%%%%%%%%%%%%%%%%%%%%%%%%%%%%%%%%%%%%%%%%%%%%%%%%%%%%%%%%%%%%%%%%%%%%%%%%%%%%%%
\section*{Acknowledgements}
We thank Francis Bach, G\'erard Ben-Arous, Lenaic Chizat, Theodor Misiakiewicz \& Lenka Zdeborov\'a for valuable discussions. We acknowledge funding from the Swiss National Science Foundation grant SNFS OperaGOST, $200021\_200390$ and the \textit{Choose France - CNRS AI Rising Talents} program.
% %%%%%%%%%%%%%%%%%%%%%%%%%%%%%%%%%%%%%%%%%%%%%%%%%%%%%%%%%%%%%%%%%%%%%%%%%%%%%%%

\bibliographystyle{unsrt}
\bibliography{bibliography}
%%%%%%%%%%%%%%%%%%%%%%%%%%%%%%%%%%%%%%%%%%%%%%%%%%%%%%%%%%%%%%%%%%%%%%%%%%%%%%%
% APPENDIX
%%%%%%%%%%%%%%%%%%%%%%%%%%%%%%%%%%%%%%%%%%%%%%%%%%%%%%%%%%%%%%%%%%%%%%%%%%%%%%%
\newpage
\appendix
%%%%%%%%%%%%%%%%%%%%%%%%%%%%%%%%%%%%%%%%%%%%%%%%%%%%%%%%%%%%%%%%%%%%%%%%%%%%%%%
\section{Effect of noise on Equations~\eqref{eq:def:gf_ode}}
\label{app:noise-term}
In this appendix we present some corrective terms to Equations~\eqref{eq:def:gf_ode} that allows to have better results when numerically integrating ODEs and comparing them to simulations.

In \cite{veiga2022phase} the terms proportional to $\sfrac{\lr}{\hids}$ are neglected completely when running numerical integrations. However, we noticed that a term from $\Psi^{(\mathrm{Var})}_{ij}(\Omega)$ could be kept in order to get better agreement between ODEs and simulation in presence of label noise at small but finite $\sfrac{\gamma}{p}\ll 1$.
Letting \(\mathcal{E}^*  \coloneqq \mathcal{E} - \sqrt{\noise}z\), we can decompose
\[
    \Psi^{(\mathrm{Var})}_{ij}(\Omega) = \mathbb E_{(\vec{\lambda},\vec{\lambda}^{\star})\sim\mathcal{N}(\vec{0}_{\hids+\hidt}, \Omega)} \left[\sigma'(\lambda_i)\sigma'(\lambda_j)\,{\mathcal{E}^*}^2 \right] + \mathbb E_{(\vec{\lambda},\vec{\lambda}^{\star})\sim\mathcal{N}(\vec{0}_{\hids+\hidt}, \Omega)} \left[\sigma'(\lambda_i)\sigma'(\lambda_j)\,\noise \right],
\]
allowing us to define
\begin{equation}
    \Psi^{(\mathrm{noise})}_{ij}(\Omega) \coloneqq \mathbb E_{(\vec{\lambda},\vec{\lambda}^{\star})\sim\mathcal{N}(\vec{0}_{\hids+\hidt}, \Omega)} \left[\sigma'(\lambda_i)\sigma'(\lambda_j)\,\noise \right].
\end{equation}
The asymptotic of this term at late times $t\gg 1$ was explicitly computed in \cite{biehl_1995} and \cite{goldt_2019} for different architectures, where it was found that $\Psi^{(\mathrm{noise})}_{ij} \propto \noise$. In these cases, when the dynamics approaches the point where \(\mathcal{E}^* \sim \sfrac{\lr\noise}{\hids}\), then the term $\Psi^{(\mathrm{GF})}_{ij}(\Omega)$ is of the same order of $\sfrac{\lr}{\hids}\Psi^{(\mathrm{noise})}_{ij}(\Omega)$, while the remaining part of $\Psi^{(\mathrm{Var})}_{ij}(\Omega)$ is still negligible since is scaling as \({\mathcal{E}^*}^2\). 
It follows that the first order correction in $\sfrac{\gamma}{p}$ of equations \eqref{eq:def:gf_ode} is given by: 
\begin{align}
\tag{GF-ODE-NOISE}\label{eq:def:ode_gf_withnoise}
\frac{\dd\mat{M}}{\dd t} = \Psi^{(\mathrm{M})}(\Omega), && \frac{\dd\mat{Q}}{\dd t} =  \Psi^{(\mathrm{GF})}(\Omega) + \frac{\lr}{\hids}\Psi^{(\mathrm{noise})}(\Omega).
\end{align}
Despite vanishing in the true $\sfrac{\gamma}{\hids}\to0^+$ limit, the new term enables ODE to catch the behaviour for large times, when simulating small but finite \(\sfrac{\gamma}{p}\). 

\begin{figure}[ht]
    \centering
    \includegraphics[width=.5\textwidth]{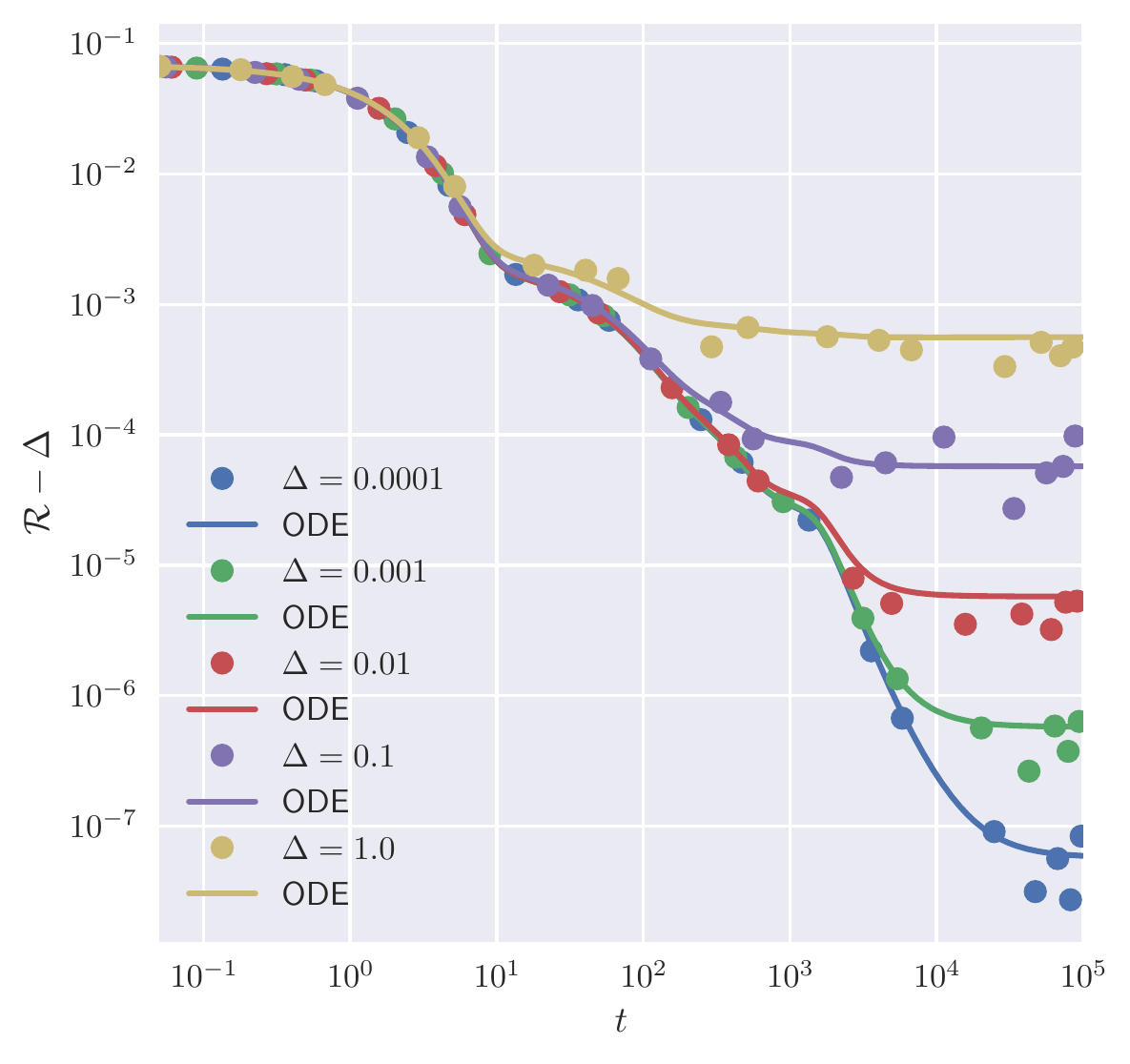}

    \caption{
    comparison between the  simulated learning dynamic and the one obtained by integrating the differential equations, in the \emph{classical regime}; \(\hidt=4,\hids=8,\lr=\num{5e-2},d=10, \sigma=\erf(\sfrac{\cdot}{\sqrt{2}})\). Learning the same teacher with different level of noise.
    }
    \label{fig:classic-limit-noise}
\end{figure}

Figure~\ref{fig:classic-limit-noise} shows a numerical experiment explicitly designed to show the effect of noise term. 
The fluctuations visible in the final plateaus result from the stochastic process in Equation~\eqref{eq:def:overlap_process}: 
\(\Psi^{(\mathrm{Var})}\) and consequently \(\Psi^{(\mathrm{noise})}\) are proportional to \(\sfrac{||\vec{x}^{\i}||_{2}^{2}}{d}\). When \(d=O(1)\) then this term is not concentrating to 1 and leads to a fluctuation. 
Let us again stress the fact that this is an effect visible only when performing numerical experiments with finite GP, whereas in the true limit the fluctuations disappear and the dynamic is described by deterministic ODEs.
% %%%%%%%%%%%%%%%%%%%%%%%%%%%%%%%%%%%%%%%%%%%%%%%%%%%%%%%%%%%%%%%%%%%%%%%%%%%%%%%
\newpage
\section{Derivation of Eq. \eqref{eq:def:mf_orthogonal} and local fields covariance}
\label{app:mean_field_orthogonal_derivation}
Let's start by taking the time derivative of Eq.~\eqref{eq:def:Qorth}
\[
    \dod{\Q^\bot}{t} = \dod{\Q}{t} - \dod{\M}{t}\P^{-1}\M^\top - \M\P^{-1}\dod{\M^\top}{t} \coloneqq \Psi^\bot(\Omega),
\]
and plugging in equations~\eqref{eq:def:gf_ode}~and~\eqref{eq:def:Psi} one after the other we get
\[\begin{split}
\Psi^\bot_{ij}(\Omega) \!&=\! \mathbb E \Bigg[\!\left(\sigma'(\lambda_i)\lambda_{j}\!+\!\sigma'(\lambda_j)\lambda_{i}\right) \mathcal{E} \!-\!
    \!\sum_{r=1}^k\!\sum_{t=1}^k\!\sigma'(\lambda_i)\lambda^{\star}_{r}\mathcal{E}\!\left[\!\P^{-1}\!\right]_{rt}\!\left[\!\M^\top\!\right]_{tj} -
    \!\sum_{r=1}^k\!\sum_{t=1}^k\!\left[\M\right]_{ir}\!\left[\!\P^{-1}\!\right]_{rt}\!\lambda^{\star}_{t}\sigma'(\lambda_i)\mathcal{E}\Bigg] \\
    &=\!\mathbb E \!\left[\!
        \mathcal{E} \sum_{r=1}^k\!\sum_{t=1}^k \left(\!
        \sigma'(\lambda_i)\left(\lambda_{j}\!-\!\lambda^{\star}_{r}\!\left[\P^{-1}\right]_{rt}\!\left[\M^\top\right]_{tj}\right) +
        \sigma'(\lambda_j)\left(\lambda_{i}\!-\!\left[\M\right]_{ir}\!\left[\P^{-1}\right]_{rt}\!\lambda^{\star}_{t}\right)
        \!\right)
    \!\right],
\end{split}\]
where all the expected value are intended over $(\vec{\lambda},\vec{\lambda}^{\star})\sim\mathcal{N}(\vec{0}_{\hids+\hidt}, \Omega)$. Starting from the definition of $\vec{\lambda}^\bot$
\[
   \vec{\lambda}^{\bot} = W^{\bot}\vec{x} = \left(W -  M P^{-1} W^\star\right)\vec{x} = \vec{\lambda} - M P^{-1} \vec{\lambda}^{\star},
\]
and writing single component
\[
   \vec{\lambda}^{\bot}_i = \lambda_{i}- \sum_{r=1}^k\sum_{t=1}^k\left[\M\right]_{ir}\left[\P^{-1}\right]_{rt}\lambda^{\star}_{t},
\]
substituting in the expression above we finally get
\begin{equation} 
\Psi^\bot_{ij}(\Omega) = \mathbb E \left[ \left(\sigma'(\lambda_i)\lambda^{\bot}_j + \sigma'(\lambda_j)\lambda^{\bot}_i\right)\mathcal{E}\right].
\end{equation}

The explicit computation of the expected value depends on the particular activation function used. Even though, the final expression can only be function of the covariance matrix entries, since all the local fields are zero-mean Gaussian variables.
We report the covariance matrix of local fields
\begin{equation} \label{eq:def:new_covariance}
\Cov\left[\vec{\lambda}, \vec{\lambda}^\bot, \vec{\lambda}^\star\right] = 
\begin{pmatrix}
\Q & \Q^\bot & \M \\
\Q^\bot & \Q^\bot & 0 \\
{\M^{\top}} & 0 & \P
\end{pmatrix},
\end{equation}
where $\Q^\bot = \Q - \M\P^{-1}\M^\top$ as defined above.
% %%%%%%%%%%%%%%%%%%%%%%%%%%%%%%%%%%%%%%%%%%%%%%%%%%%%%%%%%%%%%%%%%%%%%%%%%%%%%%%
\newpage
\section{Mean-field approximation: proof of Theorem \ref{thm:mf_approx_lowd_bound}}
\label{sec:app:mf}
\paragraph{Preliminaries} We begin by recalling the results of \cite{mei_2019}. For any distribution $\mu$ on $\mathbb R^d$, we define the network function
\begin{equation}
    \hat f_{\mu}(\vec{x}) = \int \sigma(\vec{w}^\top \vec{x}) d\mu(\vec{w}).
\end{equation}
It is easy to check that when
\[ \mu = \mu_p(W) := \frac1p \sum_{i = 1}^p \delta_{\vec{w}_i}, \]
then $\hat f_\mu = f_\Theta$, where $f_\Theta$ was defined in \eqref{eq:def:model}. The associated network risk is then given by
\begin{equation}
    \hat{\mathcal{R}}(\mu) = \mathbb E_{\vec{x} \sim \mathcal N(\vec{0}, \sfrac{1}{d} I_d)} \left[\frac12\left( f_{\Theta^\star}(\vec{x}) - \hat f_\mu(\vec{x}) \right)^2 \right].
\end{equation}
Similarly, we can consider the continuous equivalent of the gradient flow equation as follows: 
\begin{equation}
    \varphi(\vec{w}, \mu) = \mathbb E_{\vec{x} \sim \mathcal N(\vec{0}, \sfrac{1}{d} I_d)} \left[\sigma'(\vec{w}^\top \vec{x}) \vec{x} \left( f_{\Theta^\star}(\vec{x}) - \hat f_\mu(\vec{x}) \right) \right].
\end{equation}
Through conservation of matter arguments, the evolution of the empirical measure for particles following the gradient flow equation \eqref{eq:gf} obeys the following partial differential equation:
\begin{equation}
    \label{eq:def:gf_pde}\tag{GF-PDE}
    \partial_t \mu_t = \nabla_{\vec{w}} \cdot \left( \mu_t\,\varphi(\,\cdot\,, \mu_t) \right) 
\end{equation}

\begin{theorem}[\cite{mei_2019}, Propositions 13-16]\label{thm:mei_mf_bound}
    Let $\mu_0$ be a $K^2/d$-subgaussian distribution, and consider the following two processes:
    \begin{itemize}
        \item the solution $\mu_t$ of \eqref{eq:def:gf_pde} with initial value $\mu_0$,
        \item the solution $W(t)$ of the gradient flow ODE \eqref{eq:gf}, with initial conditions $\vec{w_i}(0) \sim \mu_0$ i.i.d.
    \end{itemize}
    Then for any $T \geq 0$ we have
    \begin{equation}\label{eq:app:mei_mf_bound}
        \sup_{t \in [0, T]} \left|\mathcal{R}(W(t)) - \hat{\mathcal{R}}(\mu_t) \right| \leq \frac{Ce^{CT} \left(\sqrt{\log(pT)} + z \right)}{\sqrt{p}}
    \end{equation}
    with probability at least $1 - e^{-z^2}$.
\end{theorem}

As we have already mentioned in the section devoted to the gradient flow regime, the solution $\Omega(t)$ of \eqref{eq:def:gf_ode} is exactly the overlap matrix of $W(t)$, and hence the term $\mathcal{R}(W(t))$ can be replaced by $\mathcal{R}(\Omega(t))$ in \eqref{eq:app:mei_mf_bound}.

\paragraph{Rotation invariance} The crux of Theorem \ref{thm:mf_approx_lowd_bound} lies in the rotation invariance of $\mu_0$. Indeed, as noticed in \cite{pmlr-v178-abbe22a, chizat_meanfield_symmetries}, such symmetries are conserved throughout the mean-field dynamics.
\begin{proposition}[\cite{chizat_meanfield_symmetries}, Proposition 2.1]\label{prop:app:chizat_symmetry}
Let $U$ be a linear transformation, and assume that the initial measure $\mu_0$, the teacher function $f_\star$ and the data measure $\rho$ are all invariant under $U$. Let $\mu_t$ be the solution to \eqref{eq:def:gf_pde} with initial condition $\mu_0$. Then $\mu_t$ is $U$-invariant for all $t \geq 0$.
\end{proposition}

Write $\mathbb R^d = V^\star \oplus V^\bot$, where $V^\star$ is the span of $W^\star$ and $V^\bot$ is its orthogonal subspace. Proposition \ref{prop:app:chizat_symmetry} then implies immediately that $\mu_t$ is rotation-invariant on $V^\bot$. Hence, if we define the following function:
\begin{equation}
h(\vec{w}) =  \left(\frac{W^\star \vec{w}}d, \lVert \vec{w}^\bot \rVert^2\right)
\end{equation}
where $\vec{w}^\bot$ is the projection of $\vec{w}$ on $V^\bot$, then $\mu_t$ is only determined by its pushforward $\tilde \mu_t = h_\#\mu_t$. 
Conversely, for a set of reduced parameters $(\vec{m}, q)\in\mathbb{R}^d$, and $\vec{g} \in \mathbb{S}^{d-k-1}$, we define
\begin{equation}
    \tilde{\vec{w}} = W^{\star \top}P^{-1}\vec{m} + \sqrt{q} \cdot \vec{g},
\end{equation}
where $\vec{g}$ is a uniform unit vector in $V^\bot$. Then $\mu_t$ is the measure of $\tilde{\vec{w}}$ when $(\vec{m}, q) \sim \tilde \mu_t$ and $\vec{g}\sim \mathrm{Unif}(\mathbb{S}^{d-k-1})$. This allows us to write the reduced equations for $\tilde{\mu_t}$:
\begin{equation}
\label{eq:app:df_pde}\tag{DF-PDE}
    \begin{split}
        \partial_t \tilde\mu_t &= \nabla_{(\vec{m}, q)} \cdot \left( \tilde \mu_t \tilde \varphi(\,\cdot\,, \tilde \mu_t) \right) \\
        \tilde \varphi_{\vec{m}}((\vec{m}, q), \tilde \mu) &= \mathbb{E}_{\vec{x}, \vec{g}}\left[ \sigma'(\tilde{\vec{w}}^\top \vec{x}) W^\star \vec{x} \left( f_{\Theta^\star}(\vec{x}) - \tilde f_{\tilde \mu}(\vec{x}) \right) \right] \\
        \tilde \varphi_{q}((\vec{m}, q), \tilde \mu) &= \mathbb{E}_{\vec{x}, \vec{g}}\left[ \sigma'(\tilde{\vec{w}}^\top \vec{x}) \sqrt{q}\,\vec{g}^\top \vec{x} \left( f_{\Theta^\star}(\vec{x}) -\tilde  f_{\tilde \mu}(\vec{x}) \right) \right]
    \end{split}
\end{equation}
where $\vec{x} = (\vec{z}, \vec{r}) \sim \mathcal{N}(0, \sfrac{1}{d}I_d)$ is a normalised Gaussian vector, $\vec{g} \sim \mathrm{Unif}(\mathbb{S}^{d-k-1})$, and
\begin{equation}
    \tilde f_{\tilde \mu}(\vec{x}) = \int \sigma(\tilde{\vec{w}}^\top \vec{x}) d\tilde\mu(\vec{m}, q) d\nu(\vec{g})
\end{equation}
The associated population risk is now
\begin{equation}
    \tilde{\mathcal{R}}(\tilde\mu) :=  \mathbb E_{\vec{x} \sim \mathcal N(\vec{0}, \sfrac{1}{d} I_d)} \left[\frac12 \left( f_{\Theta^\star}(\vec{x}) - \tilde f_{\tilde\mu}(\vec{x}) \right)^2 \right].
\end{equation}

We have therefore shown the following proposition:
\begin{proposition}
    Assume that $\mu_0$ is rotation-invariant on $V^\bot$. Consider these two processes:
    \begin{itemize}
        \item the solution $\mu_t$ of \eqref{eq:def:gf_pde} with initial value $\mu_0$,
        \item the solution $\tilde\mu_t$ of \eqref{eq:app:df_pde} with initial value $h_{\#}\mu_0$. 
    \end{itemize}
    Then, for all $t \geq 0$, we have
    \begin{equation}
        \hat{\mathcal{R}}(\mu_t) = \tilde{\mathcal{R}}(\tilde\mu_t)
    \end{equation}
\end{proposition}

\paragraph{Back to ODEs} Consider a population of $p$ particles $(\vec{m}, q)$ that evolve according to the following equations:
\begin{equation}
    \label{eq:app:df_ode}\tag{DF-ODE}
    \begin{split}
        \frac{d\vec{m_i}}{dt} &= \tilde \varphi_{\vec{m}}(\vec{m_i}(t), \tilde\mu_{p}(t)) \\
        \frac{dq_i}{dt}& = \tilde \varphi_{q}(q_i(t), \tilde\mu_{p}(t))
    \end{split}
\end{equation}
where $\tilde\mu_p(t)$ is the empirical distribution of the population:
\[ \tilde \mu_p(t) = \frac1p \sum_{i=1}^p \delta_{\vec{m}_i(t), q_i(t)}. \]
Then, by the same arguments as Theorem \ref{thm:mei_mf_bound}, we have:
\begin{proposition}\label{prop:mf_approx_reduced}
    Assume that $\tilde\mu_0$ is a $K^2/d$-subgaussian distribution, and consider the two processes:
    \begin{itemize}
        \item the solution $\tilde\mu_t$ of \eqref{eq:def:gf_pde} with initial value $\tilde\mu_0$,
        \item the solution $\tilde\Theta(t) = (\vec{m_i}(t), q_i(t))_i$ of \eqref{eq:app:df_ode} with initial conditions $\vec{m}_i(0), q_i(0) \sim \tilde \mu_0$ i.i.d.
    \end{itemize}
    Then for all $T \geq 0$, we have
    \begin{equation}
        \sup_{t \in [0, T]} \left|\tilde{\mathcal{R}}(\tilde\Theta(t)) - \tilde{\mathcal{R}}(\tilde\mu_t) \right| \leq \frac{Ce^{CT} \left(\sqrt{\log(pT)} + z \right)}{\sqrt{p}},
    \end{equation}
    with probability at least $1 - e^{-z^2}$, where $\tilde{\mathcal{R}}(\Theta(t)) = \tilde{\mathcal{R}}(\tilde\mu_p(t))$.
\end{proposition}

In conclusion, we have shown the following:
\begin{theorem}\label{thm:app:mf_approx_bound}
    Assume that conditions \ref{assump:activation} and \ref{assump:init} are both satisfied, and define
    \begin{equation}
       \mat{M^0} = \frac{\mat{W}^{0}{\mat{W}^{\star}}^{\top}}{d},\quad  \mat{Q^0} = \frac{\mat{W}^{0}{\mat{W}^{0}}^{\top}}{d},\quad  \vec{q}^0 = \diag{Q^0 - M^0 P^{-1} M^{0\top}}.
    \end{equation}
    Consider the two following processes:
    \begin{itemize}
        \item the solution $\Omega(t)$ of \eqref{eq:def:gf_ode} with initial value $\Omega^0$,
        \item the solution $\tilde\Theta(t)$ of \eqref{eq:app:df_ode} with initial value $(M^0, \vec{q}^0)$.
    \end{itemize}
    Then for any $T \geq 0$, we have
    \begin{equation}
        \sup_{t \in [0, T]} \left|\mathcal{R}(\Omega(t)) - \tilde{\mathcal{R}}(\tilde\Theta(t)) \right| \leq \frac{Ce^{CT} \left(\sqrt{\log(pT)} + z \right)}{\sqrt{p}},
    \end{equation}
     with probability at least $1 - e^{-z^2}$.
\end{theorem}

\paragraph{Matching the equations} To show that Theorem \ref{thm:app:mf_approx_bound} implies Theorem \ref{thm:mf_approx_lowd_bound}, we need to show the following:
\begin{itemize}
    \item the update equations for \eqref{eq:def:mf_ode} and \eqref{eq:app:df_ode} are the same
    \item the risk definitions for $\tilde{\mathcal{R}}(\tilde\Theta)$ matches the one from \eqref{eq:def:mf_risk}.
\end{itemize}
We will only show it for $\tilde\varphi_{\vec{m}}$; the rest is done similarly. Expanding the definition, we have
\begin{equation}\label{eq:app:mf_ode_expansion}
    \tilde \varphi_{\vec{m}}((\vec{m}_i, q_i), \tilde \mu_p)_r = \frac1k \sum_{s = 1}^k \langle  \sigma'(\tilde\lambda_i) \lambda^\star_r \sigma(\lambda^\star_s) \rangle - \frac1p \sum_{j=1}^p \langle  \sigma'(\tilde\lambda_i) \lambda^\star_r \sigma(\tilde\lambda_j) \rangle
\end{equation}
where $\tilde \lambda_i = \tilde{\vec{w}}_i^\top \vec{x}$ and $\langle \cdot \rangle$ denotes the expectation with respect to $\vec{x}, \vec{g}$. On the other hand,
\begin{equation}\label{eq:app:ss_ode_expansion}
    \Psi^{(M)}(\Omega) = \frac1k \sum_{s = 1}^k \langle  \sigma'(\lambda_i) \lambda^\star_r \sigma(\lambda^\star_s) \rangle - \frac1p \sum_{j=1}^p \langle  \sigma'(\lambda_i) \lambda^\star_r \sigma(\lambda_j) \rangle,
\end{equation}
without any expectation on $\vec{g}$; hence we only need to match the expressions term by term. The first sums of \eqref{eq:app:mf_ode_expansion} and \eqref{eq:app:ss_ode_expansion} are actually identical (since the marginal distribution of $\lambda_i^\star$ is independent from $\vec{g}$), so we look at the second ones:
\begin{equation}
    \langle  \sigma'(\tilde\lambda_i) \lambda^\star_r \sigma(\tilde\lambda_j) \rangle = \int \mathbb{E}_{x \sim \mathcal{N}(\vec{0}, \sfrac{1}{d}I_d)}\left[ \sigma'(\tilde\lambda_i) \lambda^\star_r \sigma(\tilde\lambda_j) \right] d\nu(\vec{g}_i) d\nu(\vec{g}_j)
\end{equation}
For $i \neq j$ and a given realization of $\vec{g}_i, \vec{g}_j$, the covariance matrix of $(\tilde \lambda_i, \tilde \lambda_j, \lambda^\star_r)$ is
\[ \mathrm{Cov}\left[\tilde \lambda_i, \tilde \lambda_j, \lambda^\star_r\right] = \begin{pmatrix}
    \vec{m}_i^\top P^{-1} \vec{m_i} + q_i & \vec{m}_i^\top P^{-1} \vec{m_j} + \sqrt{q_i q_j} \xi_{ij} & m_{ir} \\
    \vec{m}_i^\top P^{-1} \vec{m_j} + \sqrt{q_i q_j} \xi_{ij} & \vec{m}_i^\top P^{-1} \vec{m_i} + q_i &  m_{jr} \\
    m_{ir} & m_{jr} & p_{rr}
\end{pmatrix} \]
where $\xi_{ij} = \langle \vec{g_i}, \vec{g_j} \rangle$. It is easy to check that this is exactly equal to $\tilde\Omega^{ijr}$, where $\tilde\Omega$ is defined in \eqref{eq:def:overlaps_mf}. Finally, by linearity of expectation, since we never have a three-way correlation term between $\vec{g}_i, \vec{g}_j, \vec{g}_\ell$, we can consider all the $\xi_{ij}$ to be independent. This ends the proof of Theorem \ref{thm:mf_approx_lowd_bound}.
% %%%%%%%%%%%%%%%%%%%%%%%%%%%%%%%%%%%%%%%%%%%%%%%%%%%%%%%%%%%%%%%%%%%%%%%%%%%%%%%
\newpage
\section{Proofs for the high-dimensional mean-field approximation}
\label{sec:app:hd_mean_field}
\subsection{Preliminaries}
We first provide several bounds on expectations of functions of Gaussians, that will be used throughout this section. We begin by recalling the classical Gaussian Poincaré inequality (see e.g. \cite{boucheron_concentration_2013}, Theorem 3.20):
\begin{lemma}\label{lem:gauss_poincare}
    Let $f: \mathbb{R}^d \to \mathbb{R}$ be a differentiable function, and $X \sim \mathcal N(0, I)$. Then
    \[ \mathrm{Var}(f(X)) \leq \mathbb E\left[\lVert\nabla f(X) \rVert^2 \right]. \]
    If $X \sim \mathcal{N}(0, \Sigma)$, the following bounds hold instead:
    \[ \mathrm{Var}(f(X)) \leq \mathbb E\left[\langle\nabla f(X), \Sigma \nabla f(X) \rangle \right] \leq \lVert \Sigma \rVert_{\mathrm{op}} \cdot \mathbb E\left[\lVert\nabla f(X) \rVert^2 \right] . \]
\end{lemma}

Now, we provide some concentration bounds for expectations of random variables. Let $h:\mathbb{R}\to\mathbb{R}$ be a function, and $z$ a random variable; our goal is to bound
\[ \left|\mathbb{E}[h(z)] - h(\mathbb{E}[z])\right| \]
under different assumptions on $f, z$.
\begin{lemma}\label{lem:average_lipschitz_function}
    Assume that $h$ is $L$-Lipschitz, and that $z$ has a second moment. Then
\begin{equation}
        \left|\mathbb{E}[h(z)] - h(\mathbb{E}[z])\right| \leq L\sqrt{\Var(z)}
    \end{equation}
\end{lemma}
\begin{proof}
    By Jensen's inequality, we have
    \begin{align*}
        \left|\mathbb{E}[h(z)] - h(\mathbb{E}[z])\right| &\leq \mathbb{E}\left[|h(z) - h(\mathbb{E}[z])| \right] \\
        &\leq L \mathbb{E}[|z - \mathbb{E}[z]|] \\
        &\leq L\sqrt{\Var(z)}
    \end{align*}
    where the last line uses the Cauchy-Schwarz inequality.
\end{proof}

\begin{lemma}\label{lem:average_twice_diff_function}
    Assume that $h$ is twice differentiable, with $\norm{h^{(k)}}_\infty \leq K$, and that $z$ has a fourth moment. Then
    \begin{equation}
        \left|\mathbb{E}[h(z)] - h(\mathbb{E}[z])\right| \leq \frac K2 \sqrt{\mathbb{E}\left[(z-s)^4 \right]} 
    \end{equation}
\end{lemma}
\begin{proof}
    For simplicity, let $s = \mathbb{E}[z]$. By the Lagrange formula for Taylor series, we can write
    \begin{equation}
        h(x) = h(s) + (x-s)h'(s) + (x-s)^2 R(x)
    \end{equation}
    where $R$ is bounded by $K/2$. Plugging $z$ into this equation,
    \begin{equation}
        \left|\mathbb{E}[h(z)] - h(s)\right| = \left|\mathbb{E}\left[ (z-s)^2 R(z) \right] \right|
        \leq \sqrt{\mathbb{E}\left[(z-s)^4 \right] \mathbb{E}\left[ R(z)^2 \right]}
    \end{equation}
    via the Cauchy-Schwarz inequality, and the result ensues.
\end{proof}

\subsection{Proof of Lemma \ref{lem:Psi_lipschitz}}

We only show the result for $\Psi^{(M)}$; the one for $\Psi^\bot$ ensues from similar methods. Recalling the expansion  in \eqref{eq:app:ss_ode_expansion} as a sum of 3-point correlation functions, we define the following function of $3 \times 3$ matrices:
\begin{equation}
    f(\Sigma) = \mathbb{E}_{\vec{z} \sim \mathcal N(\vec{0}, \Sigma)}\left[ \sigma'(z_1)z_2\sigma(z_3). \right]
\end{equation}
Then $\Psi^{(M)}(\Omega)$ is simply an average of $f(\Sigma)$ for $\Sigma$ submatrices of $\Omega$. We first show the following bound:
\begin{lemma}\label{lem:psi_reduced_lipschitz}
    Under Assumption \ref{assump:activation}, the function $f$ satisfies the following inequality: for any $\Sigma$, $\Sigma'$ such that $\norm{\Sigma' - \Sigma}_\infty \leq K$,
    \begin{equation}
        \left|f(\Sigma') - f(\Sigma)\right| \leq C(1 + \sqrt{\Sigma_{22}})\norm{\Sigma' - \Sigma}_\infty
    \end{equation}
\end{lemma}
\begin{proof}
    Let $\Sigma$ and $\Sigma'$ be two positive semidefinite matrices, and $\delta\Sigma = \Sigma' - \Sigma$. We can write $\delta\Sigma = \delta\Sigma_1 - \delta\Sigma_2$ where both $\delta\Sigma_i$ are positive, and
    \[ \norm{\delta\Sigma_i}_{\infty} \leq C\norm{\delta\Sigma_i}_{\mathrm{op}} \leq  C\norm{\delta\Sigma}_{\mathrm{op}} \leq C'\norm{\delta\Sigma}_{\infty} \]
    using norm equivalence. Hence, we shall assume from now on that $\delta\Sigma$ is positive semidefinite, and write
    \[z' = z + \delta z\]
    where $\delta z \sim \mathcal N(\vec{0}, \delta\Sigma)$ is independent from $z$. Define
    \[ \delta\sigma = \sigma(z_3') - \sigma(z_3) \quad \text{and} \quad \delta\sigma' = \sigma'(z_1') - \sigma'(z_1) \]
    Then
    \begin{align*}
        f(\Sigma') - f(\Sigma) = \quad&\mathbb{E}\left[ \sigma'(z_1)z_2\delta\sigma + \sigma'(z_1)\delta z_2 \sigma(z_3) + \delta \sigma' z_2\sigma(z_3) \right] \\
        + \ &\mathbb{E}\left[ \delta\sigma'\delta z_2 \sigma(z_3) + \delta\sigma' z_2 \delta\sigma + \sigma'(z_1)\delta z_2 \delta \sigma \right] \\
        + \ & \mathbb{E}[\delta\sigma' \delta z_2 \delta \sigma]
    \end{align*}
    Since $\delta z$ is independent from $z$, and $\sigma$ satisfies Assumption \ref{assump:activation}, we can apply Lemma \ref{lem:average_twice_diff_function} to get
    \[ |\mathbb{E}_{\delta z_3}[\delta\sigma]| \leq C\delta\Sigma_{33}, \]
    and hence by the Cauchy-Schwarz inequality
    \[ \left|\mathbb{E}\left[ \sigma'(z_1)z_2\delta\sigma \right]\right| \leq C' \sqrt{\Sigma_{22}}\, \delta\Sigma_{33}.  \]
    A similar bound holds for the two other terms of the first line. For the second line, another application of Cauchy-Schwarz yields
    \begin{align*}
        \mathbb{E}_{\delta z}\left[\delta\sigma' z_2 \delta\sigma\right] &\leq z_2 \sqrt{\mathbb{E}[(\delta\sigma)^2]}\sqrt{\mathbb{E}[(\delta\sigma')^2]} \\
        &\leq z_2 \sqrt{\delta\Sigma_{11}\delta\Sigma_{33}}
    \end{align*}
    by a combination of Lemmas \ref{lem:gauss_poincare} and \ref{lem:average_lipschitz_function}. A similar bound holds for the third line, and it is easily checked that all of the obtained bounds are lower than the one of Lemma \ref{lem:psi_reduced_lipschitz}.
\end{proof}

Now, the expansion of $\Psi^{(M)}(\tilde\Omega) - \Psi^{(M)}(\bar\Omega)$ is an average of terms of the form $f(\tilde\Sigma) - f(\bar\Sigma)$, where $(\tilde\Sigma - \bar\Sigma)_{ij}$ is either $0$ or $\sqrt{q_{ii}^\bot q_{jj}^\bot} \xi_{ij}$. Hence, we can write each of those terms as $h(\xi_{ij}) - h(0)$, and by Assumption \ref{assump:boundedness} and Lemma \ref{lem:psi_reduced_lipschitz}, the function $h$ is Lipschitz. Lemma \ref{lem:Psi_lipschitz} then ensues from an application of Lemma \ref{lem:average_lipschitz_function}.

\subsection{Proof of Eq. \eqref{eq:mf:concentration}}
We begin by showing the following lemma. Recall the definition of $\mathcal{E}$:
\begin{equation}\label{eq:app:def_error}
    \mathcal E = \frac1k\sum_{r=1}^k \sigma(\lambda_r^\star) - \frac1p \sum_{i=1}^p \sigma(\lambda_i)
\end{equation}

\begin{lemma}\label{lem:app:gauss_poincare}
    There exists a constant $c \geq 0$ such that for any choice of $\vec{\lambda}^\star$,
    \[ \mathrm{Var}_{\vec{\lambda^\bot}}\left( \mathcal E \right) \leq \frac{c\, \lVert Q^\bot \rVert_{\mathrm{op}}}{p} \]
\end{lemma}

\begin{proof}
    We apply the Gauss-Poincaré inequality of Lemma \ref{lem:gauss_poincare} to
    \[ f(\vec{\lambda}^\bot) = \sum_{i=1}^p \sigma([M\vec{\lambda}^\star]_i + \lambda_i^\bot), \quad \text{which implies} \quad [\nabla f]_i = \sigma'(\lambda_i)\]
    Whenever $\sigma$ is Lipschitz, we thus have $\lVert\nabla f(\vec{\lambda}^\bot) \rVert^2 \leq c p$, and the lemma ensues.
\end{proof}

We are now in a position to show Eq. \eqref{eq:mf:concentration}. For brevity, we denote by $\mathbb E$ (resp. $\mathbb E_{\text{MF}}$) the expectations with respect to $\vec\lambda^\bot \sim \mathcal N(\vec{0}, Q^\bot)$ (resp. $\vec\lambda^\bot \sim \mathcal N(\vec{0}, \mathrm{diag}(Q^\bot))$). Since the marginals of both distributions are the same, and by linearity,
\[ \mathbb E\left[f(\lambda^\bot_i)\right] =  \mathbb E_{\text{MF}}\left[f(\lambda^\bot_i)\right] \quad \text{and} \quad  \mathbb E\left[\mathcal E\right] =  \mathbb E_{\text{MF}}\left[\mathcal E\right]. \]
Under the distribution $\mathcal N(\vec{0}, \mathrm{diag}(Q^\bot))$, $\lambda_i^\bot$ is almost independent from $\mathcal E$, except for the term containing $\sigma(\lambda_i)$. We can thus write, for any $\vec{\lambda}^\star$
\[ \mathbb E_{\text{MF}}\left[\mathcal E_i \lambda_i^\bot \right] = \mathbb E \left[ \sigma'(\lambda_i)\lambda_i^\bot \right] \mathbb E\left[ \mathcal E \right] - \frac1p\mathbb E_{\text{MF}}\left[ \sigma'(\lambda_i)\lambda_i^\bot \sigma(\lambda_i) \right] + \frac1p \mathbb E\left[\sigma'(\lambda_i)\lambda_i^\bot \right] \mathbb E\left[\sigma(\lambda_i) \right]  \]
Hence,
\begin{equation}\label{eq:app:expectation_difference}
\mathbb E\left[\mathcal E_i \lambda_i^\bot \right] - \mathbb E_{\text{MF}}\left[\mathcal E_i \lambda_i^\bot \right] = \mathbb E \left[ \sigma'(\lambda_i)\lambda_i^\bot \left(\mathcal E - \mathbb E\left[ \mathcal E \right]\right)\right]  - \frac1p\mathbb E\left[ \sigma'(\lambda_i)\lambda_i^\bot \left(\sigma(\lambda_i) - \mathbb E[\sigma(\lambda_i)] \right)\right] 
\end{equation} 
Now, using the Cauchy-Schwarz inequality,
\[  \left|\mathbb E \left[ \sigma'(\lambda_i)\lambda_i^\bot \left(\mathcal E - \mathbb E\left[ \mathcal E \right]\right)\right]\right| \leq \sqrt{\mathbb E \left[ \left(\sigma'(\lambda_i)\lambda_i^\bot\right)^2\right]\mathbb E \left[\left(\mathcal E - \mathbb E\left[ \mathcal E \right]\right)^2\right]}.\]
For Lipschitz $\sigma$, the first term is easily bounded by $c\, Q_{ii}^\bot$, and the second is exactly the variance computed in Lemma \ref{lem:app:gauss_poincare}. The second term in \eqref{eq:app:expectation_difference} being clearly negligible before the first, we finally get
\begin{equation}
    \left|\mathbb E\left[\mathcal E_i \lambda_i^\bot \right] - \mathbb E_{\text{MF}}\left[\mathcal E_i \lambda_i^\bot \right]\right| \leq c\, \sqrt{Q_{ii}^\bot \cdot \frac{\lVert Q^\bot \rVert_{\mathrm{op}}}{p}},
\end{equation}
and Eq. \eqref{eq:mf:concentration} ensues by taking the expectation w.r.t $\vec{\lambda}^\star$ on both sides.

% %%%%%%%%%%%%%%%%%%%%%%%%%%%%%%%%%%%%%%%%%%%%%%%%%%%%%%%%%%%%%%%%%%%%%%%%%%%%%%%
\newpage
\section{Derivation of explicit expression for the squared activation}
\label{sec:app:squared}
In this appendix we show how to derive the differential equations for the dynamics when both \(\sigma\) and \(\sigma^*\) are the square function. We will not present the \(\Psi^{(\mathrm{Var})}\) term since we never use the square activation in the high-dimensional regime.

The starting points are Equations~\eqref{eq:def:Psi} and the fact that \(\mathcal{R}=\sfrac{\mathcal{E}^2}{2}\). 
Due to the linearity of the expected value, we can reduce the expectation on products of \(\dsp\) and \(\lf\).
Let's start with the population risk. The expected values we need can be expanded to 
\[\begin{split}
  \Explf \left[  \dsp^2  \right]  =&
    \frac{1}{k^2} \sum_{r, s =1}^k \Explf\left[ \act(\lf_{r}^*) \act(\lf_{s}^*)  \right] +
  \\
  & \frac{1}{p^2} \sum_{j, l =1}^k \Explf\left[ \act(\lf_{j}) \act(\lf_{l}) \right]
  \\
  & - \frac{2}{pk}  \sum_{j=1}^p \sum_{r=1}^k \Explf\left[ \act(\lf_{j}) \act(\lf_{r}^*) \right],
  \\
  \Explf \left[ \act' (\lf_{j})  \lf_{l} \dsp \right]  =&  
  \frac{1}{k} \sum_{r' =1}^{k} \Explf  \left[\act' (\lf_{j}) \lf_l \act( \lf_{r'}^{*}  )  \right]
   \\
   &  -  \frac{1}{p} \sum_{l' = 1}^{p} \Explf \left[ \act' (\lf_{j}) \lf_l   \act ( \lf_{l'}  ) \right] \;, 
  \\
   \Explf   \left[\act' (\lf_{j}) \lf_{r}^* \dsp \right]  =&    \frac{1}{k} \sum_{r'=1}^{k}   \Explf   \left[\act' (\lf_{j}) \lf_{r}^* \act( \lf_{r'}^{*}  )  \right]
   \\
   &   -   \frac{1}{p} \sum_{l' = 1}^{p} \Explf \left[ \act' (\lf_{j}) \lf_{r}^*   \act ( \lf_{l'}  ) \right]  \;.
\end{split}\]
These expansions are still valid for any generic activation function.
Before specializing in \(\act(x)=x^2\), we introduce a shorthand in the notation.
We will use \[\omega_{\alpha\beta} \coloneqq \left[\mat{\Omega}\right]_{\alpha\beta},\]
where the indices \(\alpha\) and \(\beta\) can discriminate between teacher and student local fields, as well as the numerical index.
% maybe expand the discussion on how \alpha and \beta work
Actually, this notation allow us to compute also \(\Psi^\bot(\Omega)\) by expanding Equation~\eqref{eq:def:mf_orthogonal} as above, and using the matrix in Equation~\eqref{eq:def:new_covariance} as covariance for the normal distribution.

With this consideration, there are only 2 types of expected values to be computed.
Let us write them explicitly, using our specific activation function
\[\begin{split}
  \Explf  \left[\act(\lf^\alpha)  \act(\lf^\beta)  \right] &= 
  \Explf  \left[(\lf^\alpha)^2 (\lf^\beta)^2 \right], \\
  \Explf  \left[\act'(\lf^\alpha) \lf^\beta \act(\lf^\gamma) \right] &= 
  2\Explf  \left[\lf^\alpha \lf^\beta (\lf^\gamma)^2\right]. \\
\end{split}
\]
We are left with some expected values of polynomials of Gaussian variables. Since the local fields all have zero mean, these are nothing but moments of a Gaussian distribution with multiple variables. The standard result used to calculate these is the Isserlis’ Theorem:
\begin{align*}
  \Explf  \left[(\lf^\alpha)^2 (\lf^\beta)^2 \right] 
  &= \omega_{\alpha\alpha}  \omega_{\beta\beta} + 2 \omega_{\alpha\beta}^2, \\
  2\Explf  \left[\lf^\alpha \lf^\beta (\lf^\gamma)^2\right] 
  &= 2 \omega_{\alpha\beta}\omega_{\gamma\gamma} + 4 \omega_{\alpha\gamma}\omega_{\beta\gamma}.\\
\end{align*}
By retracing all steps backward and making the necessary substitutions,
we can arrive at an explicit form of the Equations~\eqref{eq:def:gf_ode}. In order
to obtain a matrix form such as in the Equations~\eqref{eq:squared_ODE}, 
we have to write in a closed form all the summations appeared during the derivation
and use the fact that \(\Q\) and \(\P\) are symmetric matrices.
% %%%%%%%%%%%%%%%%%%%%%%%%%%%%%%%%%%%%%%%%%%%%%%%%%%%%%%%%%%%%%%%%%%%%%%%%%%%%%%%
\newpage
\section{Numerical experiments in the mean-field limit}
\label{sec:app:numerics-mean-field}
In this appendix we show the result of some numerical experiments we performed to verify the statements exposed in Section~\ref{sec:meanfield}.  We also refer to our GitHub repository on \href{https://github.com/IdePHICS/DimensionlessDynamicsSGD}{[https://github.com/IdePHICS/DimensionlessDynamicsSGD]}.  

First, we checked that the point \(\Q^\bot = 0\) is a fixed point of the dynamic.
We compare a simulation of this case with an integration of just the matrix \(\M\). Since when \(\Q^\bot = 0\), \(\M\) is a \emph{sufficient statistic} by itself, we expect it to be the only parameter to be evolved to characterize the dynamics.
Figure~\ref{fig:mf-numerical-simulation-appendix}~(a) shows agreement between simulation and ODE integration. 

Secondary, we would like to test is the actual need to calculate expected value on the matrix \(\Xi\) when integrating the mean field in low dimension.
In Figure~\ref{fig:mf-numerical-simulation-appendix}~(b) we present the results for \(\erf\) activation function: not taking into account the off-diagonal terms, the integration of the ODEs do not match the simulated dynamics even for large \(\hids\). 
The effect becomes even more evident by comparing  Figure~\ref{fig:mf-numerical-simulation-appendix}~(c) with Figure~\ref{fig:meanfield-limit}~(a). In fact, we can see that in the case of quadratic activation, neglecting the matrix \(\Xi\) leads to a mismatch of population risk even at the initial instant, as it follows naturally from Equation~\eqref{eq:def:mf_square}.
We do not see such a marked difference in the case of \(\sigma=\erf(\sfrac{\cdot}{\sqrt2})\) since the latter is an odd function and non-diagonal terms of $\Xi$ are symmetric random variables.

Finally, we looked at the evolution of the distribution of student weights. We use use $k=1$ to have just one teacher vector $\vec{w}^{\star}\in\mathbb{R}^{d}$, while \(p=5000\) to have enough sample to understand how the distribution is evolving. In particular, we are looking at the distribution of the cosines of the angles between student weights and $\vec{w}^{\star}$; in Figure~\ref{fig:mf-numerical-simulation-appendix}~(d) we plot this distribution a 3 different times of the evolution.
At \(t=0\), the distribution is uniform, as expected since we sample weights from a spherical invariant distribution with \(d=3\). 
During the evolution, the student's weight moves in the direction of $\w^\star$ and $-\w^\star$ (since $\act$ is an even function the sign of the weights does not count). At large time, the distribution is essentially $\sfrac12\left(\delta_{\w^\star} + \delta_{-\w^\star}\right)$, which represents perfect learning.
%: $\sfrac{\M_{j1}}{\sqrt{\P_{11}\Q_{jj}}} = \sfrac{\w_j\cdot\w^*}{||\w_j|| ||\w^*||}.$
\begin{figure}[ht]
    \centering
    \begin{minipage}[t]{.49\textwidth}
        \includegraphics[width=\textwidth]{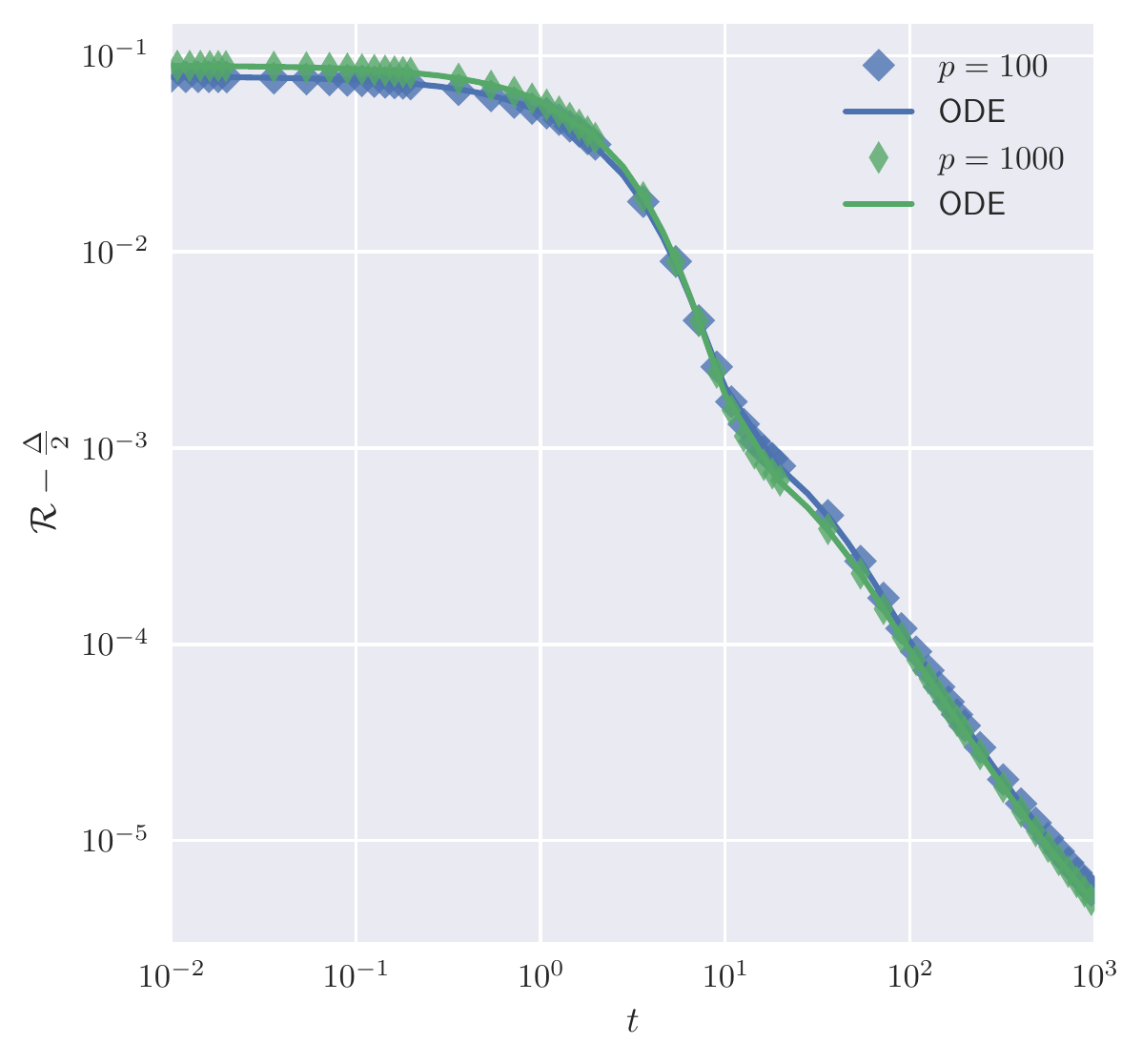}
        \begin{center}
            (a) \(\hidt=2,d=10,\lr=\num{1.},\noise=\num{0.},\sigma(x)=\erf(\sfrac{x}{\sqrt2})\)
        \end{center}
    \end{minipage}
    \begin{minipage}[t]{.49\textwidth}
        \includegraphics[width=\textwidth]{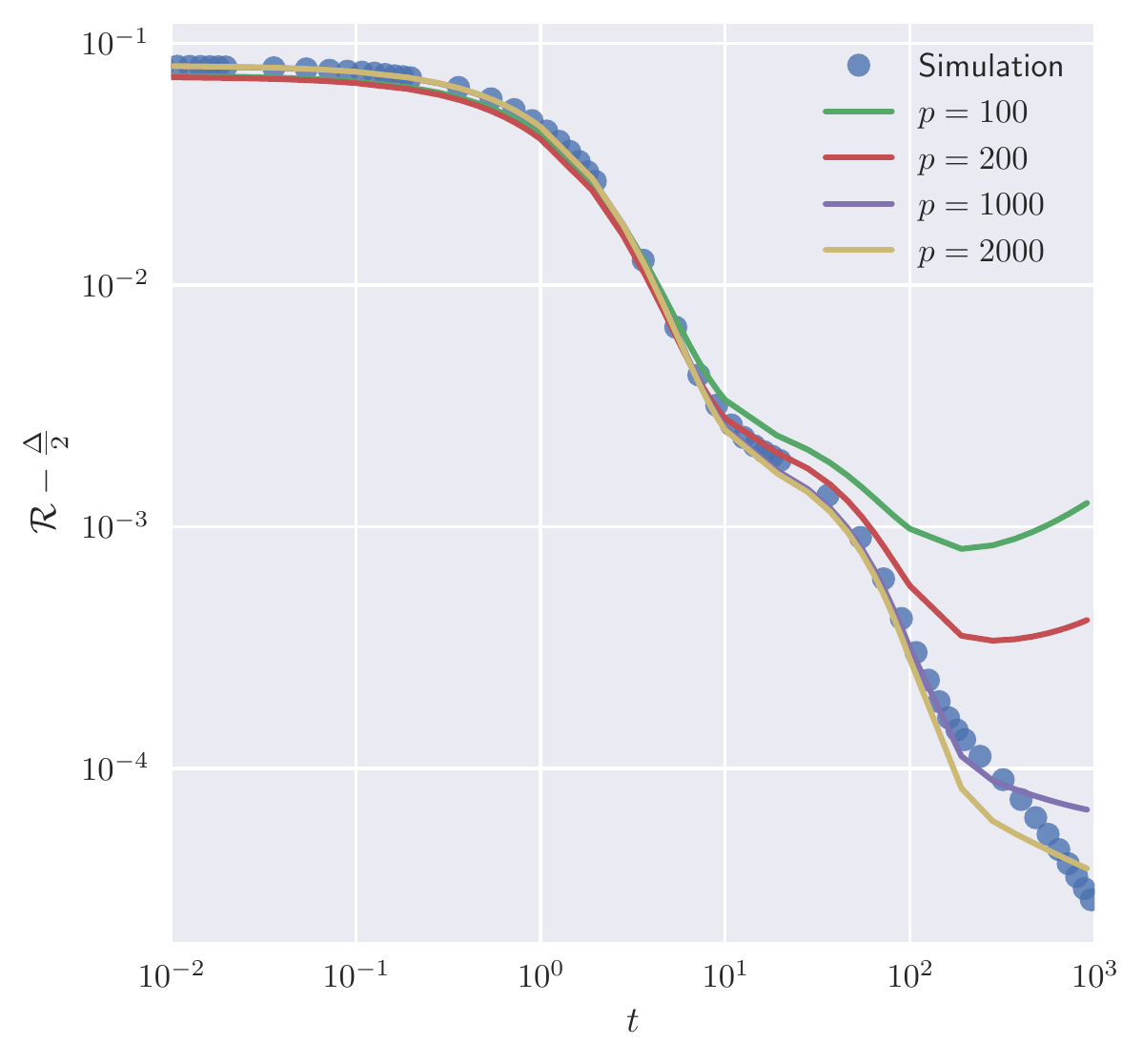}
        \begin{center}
            (b) \(\hidt=2,d=5,\lr=\num{1.},\noise=\num{0.},\sigma(x)=\erf(\sfrac{x}{\sqrt2})\)
        \end{center}
    \end{minipage}
    \begin{minipage}[t]{.49\textwidth}
        \includegraphics[width=\textwidth]{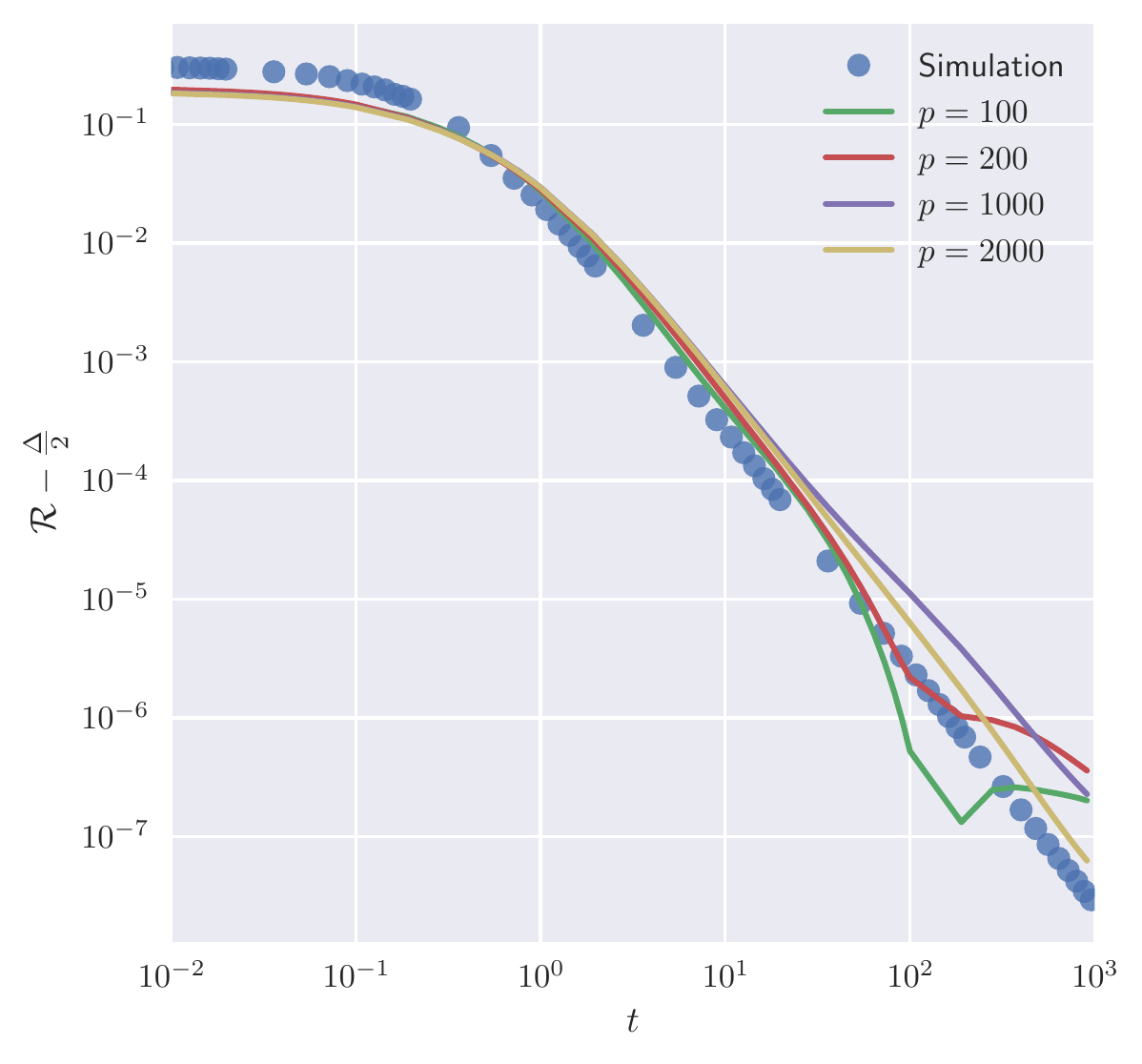}
        \begin{center}
            (c) \(\hidt=2,d=5,\lr=\num{1.},\noise=\num{0.},\sigma(x)=x^2\)
        \end{center}
    \end{minipage}
    \begin{minipage}[t]{.49\textwidth}
        \includegraphics[width=\textwidth]{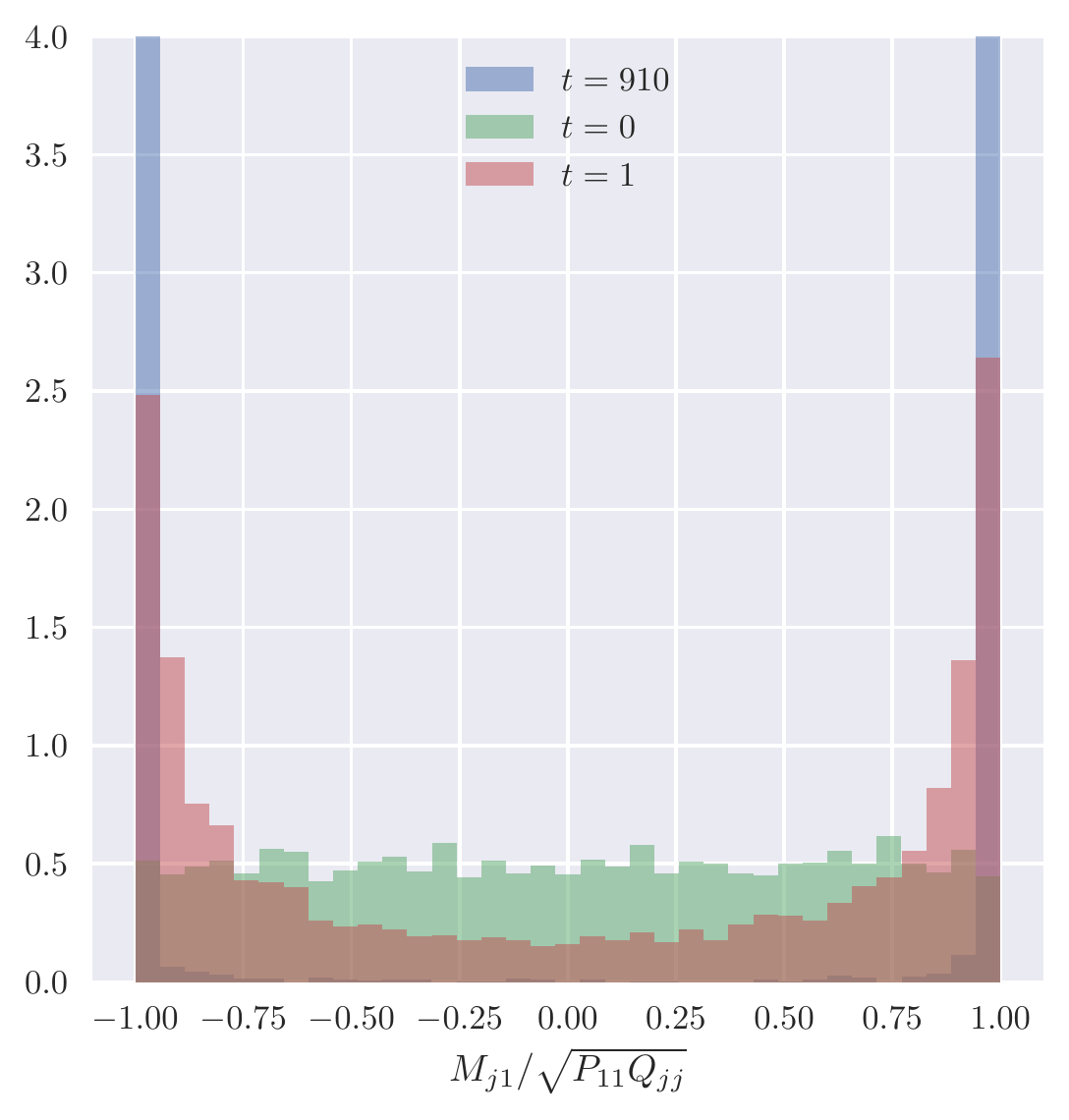}
        \begin{center}
            (d) \(\hidt=1,d=3,\lr=\num{1.},\noise=\num{0.},\sigma(x)=x^2\)
        \end{center}
    \end{minipage}

    \caption{
     (a) Evolution from starting point \(Q^\bot = 0\), tracking only \(\M\). (b) - (c)  Low-dimension simulations and the corresponding ODE trajectory, without the off-diagonal entries of matrix $\Xi$ (d) Histogram for the distribution of teacher student-teacher weights overlappings, at different times.
    }
    \label{fig:mf-numerical-simulation-appendix}
\end{figure}
%%%%%%%%%%%%%%%%%%%%%%%%%%%%%%%%%%%%%%%%%%%%%%%%%%%%%%%%%%%%%%%%%%%%%%%%%%%%%%%
\end{document}